\documentclass[lettersize,journal]{IEEEtran}
\usepackage{amsmath,amsfonts}
\usepackage{array}
\usepackage[caption=false,font=footnotesize]{subfig}

\usepackage{textcomp}
\usepackage{stfloats}
\usepackage{url}
\usepackage{verbatim}
\usepackage{graphicx}
\usepackage{booktabs}
\usepackage{cite}
\usepackage{amsmath}
\usepackage{paralist}

\usepackage{algorithmic}

\usepackage{color}
\usepackage{url}
\definecolor{indigo}{rgb}{0.0, 0.25, 0.42}
\definecolor{darkorange}{rgb}{1.0, 0.55, 0.0}
\definecolor{darkblue}{rgb}{0.122, 0.435, 0.698}

\usepackage{microtype}
\usepackage{graphicx}
\usepackage{booktabs} 

\usepackage[lined,boxed,commentsnumbered,ruled]{algorithm2e}
\usepackage{algorithmic}

\usepackage{amsmath}
\usepackage{amssymb}
\usepackage{mathtools}
\usepackage{amsthm}
\usepackage{stfloats}

\usepackage{amsmath}
\usepackage{amssymb}
\usepackage{mathtools}
\usepackage{amsthm}
\usepackage{stfloats}
\usepackage{hyperref}
\usepackage[capitalize,noabbrev]{cleveref}

\newtheorem{corollary}{Corollary}
\newtheorem{definition}{Definition}
\newtheorem{remark}{Remark}
\newtheorem{theorem}{Theorem}
\newtheorem{lemma}{Lemma}
\newtheorem{assumption}{Assumption}

\begin{document}

\title{GP-FL: Model-Based Hessian Estimation for Second-Order Over-the-Air Federated Learning}

\author{
Authors
\thanks{
The results of this study have been submitted in part to 2025 IEEE International Conference on Acoustics, Speech, and Signal Processing (ICASSP).
}
}
\author{Shayan Mohajer Hamidi,~\IEEEmembership{Student Member,~IEEE}, Ali Bereyhi,~\IEEEmembership{Member,~IEEE}, Saba Asaad, ~\IEEEmembership{Member,~IEEE},\\ and H. Vincent Poor, ~\IEEEmembership{Life Fellow,~IEEE}.
\thanks{
Shayan Mohajer Hamidi is with the Department of Electrical and Computer Engineering, University of
Waterloo, Waterloo, ON N2L 3G1, Canada (email: smohajer@uwaterloo.ca).
\\
Ali Bereyhi is with the Department of Electrical
and Computer Engineering (ECE), University of Toronto, Toronto, M5S 2E4,
Canada (e-mail: ali.bereyhi@utoronto.ca).\\
Saba Asaad is with the Department
of Electrical Engineering and Computer Science, York University,
Toronto, ON M3J 1P3, Canada (e-mail: asaads@yorku.ca).\\
H. Vincent Poor is with the Department of Electrical and Computer
Engineering (ECE), Princeton University, Princeton, NJ 08544 USA (e-mail:
poor@princeton.edu).}
}



\maketitle

\begin{abstract}
Second-order methods are widely adopted to improve the convergence rate of learning algorithms. In federated learning (FL), these methods require the clients to share their local Hessian matrices with the parameter server (PS), which comes at a prohibitive communication cost. A classical solution to this issue is to approximate the global Hessian matrix from the first-order information. Unlike in idealized networks, this solution does not perform effectively in over-the-air FL settings, where the PS receives noisy versions of the local gradients. This paper introduces a novel second-order FL framework tailored for wireless channels. The pivotal innovation lies in the PS's capability to directly estimate the global Hessian matrix from the received noisy local gradients via a non-parametric method: the PS models the unknown Hessian matrix as a Gaussian process, and then uses the temporal relation between the gradients and Hessian along with the channel model to find a stochastic estimator for the global Hessian matrix. We refer to this method as \textbf{G}aussian \textbf{p}rocess-based Hessian modeling for wireless \textbf{FL} (GP-FL) and show that it exhibits a linear-quadratic convergence rate. Numerical experiments on various datasets demonstrate that GP-FL outperforms all classical baseline first and second order FL approaches. 
\end{abstract}


\section{Introduction}
\label{Sec:Introduction}
\IEEEPARstart{T}{raditionally}, the training of machine learning (ML) models has followed a centralized approach, with the training data residing in a data center or cloud parameter server (PS). However, in numerous contemporary applications, there is a growing reluctance among devices to share their private data with a remote PS. Addressing this challenge, federated learning (FL) has emerged as a viable solution \cite{mcmahan2017communication}. FL allows devices to actively contribute to the training of a global model by leveraging only their local datasets while being facilitated by a central PS. In the FL framework, devices transmit solely their local updates to the PS, thus sidestepping the need to disclose raw datasets. The process unfolds in several steps: initially, the PS disseminates the current global model parameters to all participating devices. Subsequently, each device conducts local model training based on its unique dataset and transmits the resulting local updates back to the PS. In the final step, the PS aggregates these local updates, producing a new global model parameters for the subsequent iteration of distributed training.

Considering the large sizes of the model updates, the iterative communication between the PS and clients incurs a significant communication cost. This can significantly decelerate the convergence of FL, since the communication are typically rate-limited in practice \cite{amiri2020federated,hamidi2019systems,10487854}. Several lines of research have been dedicated to theoretically characterizing the fundamental trade-offs between learning performance and communication cost, as well as improving the communication efficiency, within the FL framework \cite{braverman2016communication,han2018geometric}. 

\subsection{Communication-Efficient FL}
Early efforts to improve the communication efficiency in the FL framework can be categorized roughly into two approaches: \textbf{($i$)} the first approach tries to reduce the communication overhead per round. One way to achieve this involves utilizing quantization \cite{10091800} and sparsification \cite{bereyhi2024sparse} techniques. These techniques aim to diminish the transmitted bits and eliminate redundant parameter updates. 
In general, these approaches come with the trade-off of potentially degrading the model performance, and they must take into account the compatibility for the aggregation operation in FL \cite{wang2021field}. \textbf{($ii$)} The second approach intends to minimize the total communication rounds. 
The most well-known example is the federated averaging (FedAvg) algorithm, in which clients execute multiple local iterations before sharing their models with the PS \cite{mcmahan2017communication}, which can significantly reduce the total number of rounds. Some variants of the gradient descent (GD) method have further demonstrated the ability to substantially reduce the overall communication rounds compared to the naive GD \cite{reddi2020adaptive,tong2020effective}.

These previous approaches ignore the properties of the communication network, and look at the communication links as rate-limited orthogonal noiseless channels. Nevertheless, later studies have shown that using the properties of the underlying network, further gain in terms of communication efficiency can be achieved \cite{amiri2020federated}. A well-known case is wireless FL, in which clients can use the superposition property of multiple access channel (MAC) to perform the computation directly over the air and significantly reduce the communication costs \cite{yang2020federated}. In fact, in wireless networks the abstract view on communication links corresponds to the conventional digital transmission, where coded symbols are sent using techniques such as orthogonal frequency-division multiplexing (OFDM). For wireless FL, however, the clients can send their local parameters simultaneously via analog transmission, and the PS can apply the concept of over-the-air computation (AirComp) to aggregate the global parameter. This method allows aggregation to be performed directly over the air, significantly reducing the required bandwidth \cite{yang2020federated,bereyhi2023FL}. This approach has gained traction in wireless FL as an effective strategy for mitigating communication costs \cite{zhu2019broadband, amiri2020machine, sery2020analog, 9833972, liu2020privacy, elgabli2021harnessing}.

\subsection{Second-Order Wireless FL}
A standard framework for wireless FL consider first-order algorithms, where the clients transmit their local gradients, and the PS aggregates them gradient directly over the air using the AirComp technique. We refer to this framework as \textit{first-order AirComp}. Similar to other first-order methods, first-order AirComp at best achieves linear convergence, which leads to a relatively large number of iteration rounds to reach the desired accuracy. We further note that the analog nature of AirComp-aided computation leads to extra aggregation error from the channel. This can be seen as zero-mean fluctuations, which adds to the noiseless aggregated gradient, i.e., the estimator that PS computes in perfect FL settings. As a result, the estimator of global gradient computed in first-order AirComp has higher variance as compared with FL with noiseless aggregation. This leads to further degradation in terms of convergence behavior.


One potential solution is to incorporate second-order methods, e.g., Newton-type methods, into the wireless FL setting, as explored in \cite{9810113}. Unlike first-order methods, second-order approaches benefit from a \textit{quadratic} convergence rate by computing the update direction using both first and second-order derivatives of the loss function. The update direction computed in these methods is often called \textit{Newton direction}. 

Integrating second-order information into FL requires further access to the second-order information at the PS. This information is obtained by clients sharing their local Hessian matrices with the PS, which imposes a significant communication burden, especially in wireless settings with strictly-restricted links. To mitigate this issue, various studies have explored ways to approximate the Hessian matrix from the first-order information and/or reduced second-order information. For instance, GIANT \cite{wang2018giant} estimates the Newton direction by utilizing the global gradient and local Hessians on each client, combined with a global backtracking line search across all clients. Although this approach reduces the communication overhead, it still requires partial sharing of second-order information resulting in higher communication load as compared to first-order methods. 

The study in \cite{zhang2015disco} suggests an alternative approach, in which each client the computation of Newton direction is offloaded to the clients: the PS and devices invoke the conjugate gradient descent algorithm to compute iteratively an estimate of the global Newton direction. Although the information shared per communication round in this approach is the same as that of the first-order methods, a single iteration of gradient descent (on the model parameters) requires at least two rounds of communication. To circumvent further this overhead, the recent study in \cite{ghosh2020distributed} proposes a novel second-order method that eliminates the aggregation of local gradients, enabling a single communication round per iteration. Motivated by this work, the authors in \cite{9810113} develop a second-order wireless FL method, in which the clients compute their Newton directions locally and share them with the PS. The PS then estimates the global Newton direction from the aggregation of these local directions. This approach enables devices to communicate with the PS only once per iteration, reducing the communication overhead to that of the first-order methods. 

As a natural extension to available proposals, some recent studies have proposed integration of communication-efficient second-order FL algorithms into the analog computation framework \cite{wang2018giant,9810113}. Experiments however show that, unlike their first-order counterparts, these \textit{second-order AirComp} algorithms struggle to perform adequately. This observation can be intuitively explained as follows: the noisy aggregation of the second-order information in AirComp-aided approaches introduces significant biases to the estimate of the Newton direction, leading to suboptimal performance. The investigations reveal that such biases can severely deteriorates the convergence, and hence the overall efficiency, of these methods \cite{wang2018giant,9810113}. This study aims to develop a novel second-order AirComp algorithm that addresses this challenge in a communication-efficient way.

\subsection{Motivation and Contributions}
This work proposes a novel second-order AirComp scheme for FL which estimates the Newton direction directly from the aggregated first-order information. It is hence extremely communication-efficient, in the sense that it does not require the exchange of local Hessians or a functions of them: in each round, the clients leverage AirComp to share \textit{only} their local gradients. The PS then estimates the global Hessian matrix from a finite sequence of its noisy aggregations. For Hessian estimation, we deviate from the classical deterministic schemes and develop a nonparametric model-based estimator. The proposed estimator assumes a Gaussian prior for the Hessian matrix and determines its posterior distribution conditional to a window of most recent noisy aggregations. It then computes an unbiased estimator of the Newton direction by sampling from the posterior distribution. We refer to this method as \textbf{G}aussian \textbf{p}rocess-based Hessian modeling for \textbf{FL} (GP-FL). 
Our analysis and numerical experiments demonstrate that the scheme can efficiently suppress the undesired directional bias.

In summary, the contributions of this paper are three-fold: 

\noindent $\bullet$ We introduce GP-FL, a novel second-order AirComp algorithm for wireless FL. The key innovation lies in the PS's ability to estimate global Hessian matrix based on the received noisy sum of the gradients. This algorithm represents a fundamental departure from most existing works, which typically focus solely on stochastic gradient (SGD) during training. The incorporation of second-order information markedly diminishes the total communication rounds in AirComp-based FL, thereby enhancing communication efficiency even further.

\noindent $\bullet$ We analyze the convergence behavior of GP-FL and show that it achieves a linear-quadratic convergence rate. Specifically, the number of communication rounds required to reach an \(\varepsilon\)-accurate solution, \(T_{\varepsilon}\), is either \(\mathcal{O} \left( \log \log {1}/{\varepsilon}\right)\) or \(\mathcal{O} \left( \tfrac{\log {1}/{\varepsilon}}{\log{1/\mu}}\right)\). This contrasts with GD-based FL methods, which typically have a linear convergence rate of \(T_{\varepsilon} = \mathcal{O}({1}/{\varepsilon})\).

\noindent $\bullet$ Through extensive experiments on datasets---including three from the LIBSVM library \cite{chang2011libsvm}, Fashion-MNIST \cite{xiao2017fashion}, and CIFAR-\{10,100\} \cite{krizhevsky2009learning}---we show that GP-FL outperforms existing first- and second-order algorithms. The method is evaluated under varying levels of heterogeneity \cite{10619204}, and the impact of key hyperparameters is analyzed.

\subsection{Related Works} 
The FedAvg scheme relies solely on first-order gradients for updates \cite{mcmahan2017communication}, offering significantly faster computation compared to Hessian-based methods, especially in settings with fast communication. Adaptive step-size methods like AdaGrad \cite{duchi2011adaptive} and ADAM \cite{kingma2014adam} have also been adapted to distributed settings, showing improved convergence \cite{karimireddy2020mime}. The adaptive step-size can be represented as a diagonal matrix \(\bold{D}\), providing an alternative preconditioning to the Newton direction \(\bold{H}^{-1} \nabla f(\boldsymbol{\theta})\), where the direction is computed as \(\bold{D}^{-1} \nabla f(\boldsymbol{\theta})\).

The second-order methods studied in the literature can be categorized into two groups: ($i$) those that utilize second-order information implicitly \cite{shamir2014communication,li2019feddane,reddi2016aide}, and ($ii$) those that explicitly compute them \cite{zhang2015disco,wang2018giant,gupta2021localnewton}. DANE \cite{shamir2014communication} computes a mirror descent update on the local loss functions, which is equivalent to the GIANT update for a quadratic function. The study in \cite{li2019feddane} proposes FedDANE as a variant of DANE, specifically tailored for FL, where it uses FedAvg as a baseline with 20 local epochs and observes no improvement with their proposed method. AIDE, proposed in \cite{reddi2016aide}, is an alternative accelerated inexact version of DANE. Another line of study in the first group, i.e., group $i$, considers employing distributed quasi-Newton methods \cite{agarwal2014reliable}. CoCoA \cite{smith2018cocoa} and its trust-region extension \cite{duenner2018trust} also perform local steps on a second-order local subproblem, but they specifically address the special case of generalized linear model objectives.

The second group of studies employ the Hessian by computing it indirectly through the use of the so-called Hessian-free optimization approach. In DiSCO \cite{zhang2015disco}, the clients compute the Hessian-vector products, and subsequently the PS executes the conjugate gradient method \cite{hestenes1952methods}. This process entails one communication round for each iteration of the conjugate gradient method. GIANT \cite{wang2018giant} and LocalNewton \cite{gupta2021localnewton} both employ the conjugate gradient method at the clients: GIANT utilizes the global gradient, while LocalNewton uses the local gradients. The study in \cite{islamov2021distributed} and its FL extension \cite{safaryan2021fednl} iteratively approximate the global Hessian, requiring a similar number of communication rounds as GIANT but achieving better convergence rates. However, they lack experimental comparisons with FedAvg using multiple local steps.

In general, the performance of mentioned methods tends to degrade when integrated into the AirComp framework, in which the PS receives a distorted and noisy aggregation. The proposed scheme in this work, GP-FL, addresses this issue by explicitly accounting for the channel imperfections.

\subsection{Notation} \label{sec:not}
The gradient and Hessian of a function $f(\cdot)$ are denoted by $\nabla f(\cdot)$ and $\nabla^2 f(\cdot)$, respectively. The operators $(\cdot)^{\mathsf{T}}$ and $(\cdot)^{\mathsf{H}}$ denote the transpose and Hermitian transpose, respectively; and $\mathbb{E}(\cdot)$ represents the mathematical expectation. The notation $\mathbf{x}[j]$ refers to the $j$-th entry of vector $\mathbf{x}$. The notation $\mathcal{GP}(\mu,\kappa)$ denotes a Gaussian process with mean $\mu$ and covariance $\kappa$.

For an integer $K$, $[K]$ denotes $\{1,\cdots,K\}$. For a scalar or function $f$, $\{f_k\}_{k \in [K]}=\{f_1,\dots,f_K\}$, and $[f_k]_{k \in [K]}=[f_1,\dots,f_K]^{\mathsf{T}}$. Scalars are denoted by non-boldface letters (e.g. $a$), vectors by boldface lowercase letters (e.g. $\bold{a}$), and matrices by boldface uppercase letters (e.g. $\bold{A}$). The $j$-th entry of vector $\bold{a}$ is denoted by $\bold{a}[j]$. The matrix $\bold{I}$ represents the identity matrix. For two matrices $\bold{A}$ and $\bold{B}$ of the same size, $\bold{A} \preceq \bold{B}$ implies that $\bold{B} - \bold{A}$ is positive semi-definite.

\subsection{Organization}
The remainder of this paper is organized as follows: 
Section \ref{sec:prelem} provides an in-depth discussion of the preliminary knowledge on Newton and quasi-Newton methods. Section \ref{sec:method} formulates the problem. Section \ref{sec:gaus} presents the proposed stochastic approach for estimating the quasi-Newton direction. Section \ref{sec:algconv} provides the convergence analysis for GP-FL. Numerical results and comparison with baselines are given in Section \ref{sec:exp}. Finally, Section \ref{sec:conclusion} concludes the paper. 

\section{Preliminaries} \label{sec:prelem}
Consider a wireless network with $K$ clients and a single PS.  The clients aim to jointly train a common model for a supervised learning task over their local labeled datasets. Let $\mathcal{D}_k$ denote the local dataset at client $k$ which contains $|{\mathcal{D}_k}|$ independently-collected training tuples of the form $(\bold{u}, v)$ with $\bold{u}\in \mathbb{R}^m$ representing a sample input to the model and $v$ denoting its ground-truth label. 

The clients agree on a global model with $d$ learnable parameters. Let vector $\boldsymbol{\theta}\in \mathbb{R}^d$ represent the collection of these parameters into a vector. The ultimate goal is to train this common model by minimizing the global empirical loss that is defined as $f(\boldsymbol{\theta}) \triangleq  \frac{1}{|\mathcal{D}|} \sum_{(\bold{u}, v) \in \mathcal{D}}  \ell (\boldsymbol{\theta}, \bold{u}, v )$, where $\mathcal{D}$ denotes the global dataset, i.e., $\mathcal{D}= \bigcup_{k =1}^{K} \mathcal{D}_k$ and $\ell (\boldsymbol{\theta}, \bold{u}, v )$ is the loss function measuring the prediction error of the model with parameters $\boldsymbol{\theta}$ on input sample $\bold{u}$ relative to the ground-truth $v$.

With the dataset distributed among clients, the global training problem is expressed in terms of local empirical losses. The local empirical loss at client $k$ is $f_k\left(\boldsymbol{\theta}\right) \triangleq \frac{1}{|\mathcal{D}_k|} \sum_{(\bold{u}, v) \in \mathcal{D}_k} \ell (\boldsymbol{\theta}, \bold{u}, v )$. The global empirical loss is $f(\boldsymbol{\theta}) = \frac{1}{|\mathcal{D}|} \sum_{k \in \mathcal{S}} |\mathcal{D}_k| f_k\left(\boldsymbol{\theta} \right)$, where $\mathcal{S} \subseteq [K]$ is the set of participating devices. The distributed training problem is formulated as
\begin{align} \label{eq:FL}
 \min_{\boldsymbol{\theta}} \frac{1}{|\mathcal{D}|} \sum_{k \in \mathcal{S}} |\mathcal{D}_k| f_k \left(\boldsymbol{\theta}\right),
\end{align} 
where $f_k \left(\boldsymbol{\theta} \right)$ is computed \textit{locally} at client $k$.

\subsection{First-Order Federated Learning} \label{sec:fed}
The de-facto solution to problem \eqref{eq:FL} is to employ the distributed (stochastic) gradient descent method (DGD/DSGD). In the $t$-th round of DGD, the PS shares parameter $\boldsymbol{\theta}_{t}$ with all clients. 
Each client computes its \textit{local gradient} at $\boldsymbol{\theta}_{t}$, i.e., $\nabla f_k(\boldsymbol{\theta}_{t})$,
and returns it to the PS. The PS then aggregates these \textit{local gradients} into a \textit{global gradient} as
\begin{align} 
\nabla f(\boldsymbol{\theta}_{t})= \frac{ |\mathcal{D}_k|}{|\mathcal{D}|} \sum_{k \in \mathcal{S}_t} \nabla f_k(\boldsymbol{\theta}_{t}),
\end{align}
and performs one step of gradient descent, i.e., it computes $\boldsymbol{\theta}_{t+1} = \boldsymbol{\theta}_{t} - \eta_t \nabla f(\boldsymbol{\theta}_{t})$, where $\eta_t>0$ is the global learning rate at the $t$-th round. First-order methods often converge slowly with linear (or sub-linear) convergence rates \cite{deng2021local,sharma2022federated}, requiring many iterations to approach a local minimum. Since the number of communication rounds in FL corresponds to iterations, this increases communication overhead. 


\subsection{Newton's Method in Federated Learning} \label{sec:newt}
A method to improve FL communication efficiency is to use faster-converging optimizers, such as second-order methods. These methods estimate the loss landscape's local curvature, enabling faster and more adaptive updates. While computationally intensive per iteration, they require fewer iterations to converge. Second-order methods are Newton-type techniques. Specifically, Newton's method updates the model using the second-order Taylor series approximation of the empirical loss. For a small $\mathfrak{d} \in \mathbb{R}^d$, the Taylor series expansion of $f(\boldsymbol{\theta} + \mathfrak{d})$ around $\boldsymbol{\theta}$ up to the quadratic term is given by $f(\boldsymbol{\theta}+\mathfrak{d}) \approx   f(\boldsymbol{\theta}) + \mathfrak{d}^T \nabla f(\boldsymbol{\theta}) + \frac{1}{2} \mathfrak{d}^T \nabla^2 f(\boldsymbol{\theta}) \mathfrak{d}$, with $\nabla^2 f(\boldsymbol{\theta})$ being the Hessian matrix of $f$ computed at $\boldsymbol{\theta}$. Assuming the Hessian is positive definite, the Newton direction that minimizes this quadratic approximation is given by $\mathfrak{d}^\star= -(\nabla^2 f(\boldsymbol{\theta}))^{-1} \nabla f(\boldsymbol{\theta})$. Newton's method hence updates the model parameters by stepping proportional to $\mathfrak{d}^\star$.

Deploying Newton's method for the distributed training task in \eqref{eq:FL}, the PS needs to update the global model as
\begin{subequations}
	\begin{align} \label{eq:realnew}
\boldsymbol{\theta}_{t+1} &=   \boldsymbol{\theta}_{t} - \eta_t \Big(\nabla^2 f(\boldsymbol{\theta}_{t}) \Big)^{-1} \nabla f(\boldsymbol{\theta}_{t}) \\ 
&= \boldsymbol{\theta}_{t} - \eta_t \Big( \sum_{k \in \mathcal{S}_t} \nabla^2 f_k(\boldsymbol{\theta}_{t}) \Big)^{-1}     \sum_{k \in \mathcal{S}_t} \nabla f_k(\boldsymbol{\theta}_{t}).
\end{align}
\end{subequations}
This requires transmitting local Hessians to the PS, creating a trade-off: second-order methods reduce communication rounds via faster convergence but increase communication costs per round due to larger transmission. Standard analyses show that the naive use of Newton's method in distributed settings often leads to inefficiency, as communication overhead dominates.

\subsection{Quasi-Newton Search Directions}
To address issues with Newton's method in distributed settings, quasi-Newton methods provide an alternative by avoiding explicit computation of local Hessians. Instead of $\nabla^2 f(\boldsymbol{\theta})$ in \eqref{eq:realnew}, an approximation matrix $\bold{B}_{\boldsymbol{\theta}}$ is built using the sequence of local gradients from previous iterations. This matrix is updated iteratively to include new information. To understand the quasi-Newton method, assume $f(\cdot)$ is twice continuously differentiable. We can hence write
\begin{align} \label{eq:taylor}
\nabla f(\boldsymbol{\theta} +\mathfrak{d}) =  \nabla f(\boldsymbol{\theta}) + \int_0^1 \nabla^2 f(\boldsymbol{\theta} +\tau \mathfrak{d}) \mathfrak{d} d \tau,
\end{align}
By adding and subtracting the term $\nabla^2 f(\boldsymbol{\theta})\mathfrak{d}$ to the right-hand side (r.h.s.) of \eqref{eq:taylor}, we have
\begin{align} \label{eq:taylor2}
\nabla f(\boldsymbol{\theta} +\mathfrak{d}) &=  \nabla f(\boldsymbol{\theta}) + \nabla^2 f(\boldsymbol{\theta})\mathfrak{d} \nonumber \\
 &+  \int_0^1 \big[ \nabla^2 f(\boldsymbol{\theta} +\tau \mathfrak{d})-\nabla^2 f(\boldsymbol{\theta}) \big] \mathfrak{d} d \tau.
\end{align}
Note that since the gradient function $\nabla f(\cdot)$ is continuous,  the magnitude of the integral on the r.h.s. of \eqref{eq:taylor2} is $\mathcal{O} (\Vert \mathfrak{d} \Vert )$. We now set $\boldsymbol{\theta}=\boldsymbol{\theta}_{t-1}$ and $\mathfrak{d}= \boldsymbol{\theta}_{t}-\boldsymbol{\theta}_{t-1}$. This leads to
\begin{align} 
\nabla f(\boldsymbol{\theta}_{t}) - \nabla f(\boldsymbol{\theta}_{t-1})  \label{eq:O}
= \nabla^2 f(\boldsymbol{\theta}_{t-1}) \left( \boldsymbol{\theta}_{t} - \boldsymbol{\theta}_{t-1}\right)+ \boldsymbol{\varepsilon},
\end{align}
with $\boldsymbol{\varepsilon} = \mathcal{O} (\|\boldsymbol{\theta}_{t} - \boldsymbol{\theta}_{t-1} \|)  $. Note that when $\boldsymbol{\theta}_{t-1}$ and $\boldsymbol{\theta}_{t}$ reside in a region close to the minimizer, the r.h.s. of  \eqref{eq:O} is dominated by the first term. We can hence write
\begin{align} \label{eq:secorder}
	\nabla f(\boldsymbol{\theta}_{t}) - \nabla f(\boldsymbol{\theta}_{t-1}) \approx
\nabla^2 f(\boldsymbol{\theta}_{t-1}) \left( \boldsymbol{\theta}_{t} - \boldsymbol{\theta}_{t-1}\right)   .    
\end{align}

The approximation in \eqref{eq:secorder} suggests estimating the Hessian matrix using a matrix $\bold{B}_t$ that satisfies the relation in \eqref{eq:secorder} as an identity. Specifically, by defining $\bold{w}_t = \boldsymbol{\theta}_t - \boldsymbol{\theta}_{t-1}$ and $\bold{y}_t = \nabla f(\boldsymbol{\theta}_t) - \nabla f(\boldsymbol{\theta}_{t-1})$, the quasi-Newton method computes Newton's direction using $\bold{B}_t$ that satisfies the following equation, commonly referred to as the secant equation:
\begin{align} \label{eq:B}
\bold{B}_t \bold{w}_{t} = \bold{y}_t,
\end{align}
which is an estimator of the Hessian matrix. A classic solution to \eqref{eq:B} given by Broyden–Fletcher–Goldfarb–Shanno (BFGS) method \cite{nocedal1999numerical}, which iteratively updates matrix $\bold{B}_{t+1}$ using information from ${\bold{B}_{t},\bold{w}_{t}, \bold{y}_t}$ according to $\bold{B}_{t+1}= \bold{B}_{t} - \frac{ \bold{B}_{t} \bold{w}_{t} \bold{w}_{t}^{\mathsf{T}} \bold{B}_{t}}{\bold{w}_{t}^{\mathsf{T}} \bold{B}_{t} \bold{w}_{t}}   + \frac{\bold{y}_t\bold{y}_t^{\mathsf{T}}}{\bold{w}_{t}^{\mathsf{T}}\bold{y}_t}$, ensuring the positive-definiteness of $\bold{B}_{t}$ and consequently making the following quasi-Newton search direction a descent direction:
\begin{align} \label{eq:quasi-dir}
\mathfrak{d}_t = -(\bold{B}_t)^{-1} \nabla f(\boldsymbol{\theta}_{t}).  
\end{align}

Unlike Newton's method, the quasi-Newton approach relies only on computed gradients, reducing communication requirements. However, this comes with higher variance in estimating the optimal direction. Studies show that the quasi-Newton method strikes a good trade-off, achieving faster convergence than first-order methods with similar communication costs.

\begin{remark}
Note that using BFGS involves only a rank-one update, allowing us to efficiently compute the inverse $(\bold{B}_t)^{-1}$ as required in \eqref{eq:quasi-dir} via the Sherman-Morrison formula \cite{shermen1949adjustment}.    
\end{remark}

\subsection{GP-FL: Quasi-Newton Method in Wireless Networks}
We now develop the second-order algorithm \textit{GP-FL} for FL in wireless networks based on the quasi-Newton method. The algorithm uses two key notions to adapt the quasi-Newton search directions approach to wireless networks:
\begin{inparaenum}
	\item[($i$)] it uses \textit{AirComp} to realize the aggregation directly over the air, and 
	\item[($ii$)] it models the quasi-Newton estimator of the Hessian matrix as a Gaussian process and employs the maximum-likelihood method to estimate it from noisy aggregations.
\end{inparaenum}

In the next sections, we present GP-FL, sketch its derivation and analyze its convergence. For the sake of simplicity, we use the following notation hereafter in the paper: we denote the local gradient of client $k$ in iteration $t$ as $\bold{g}_{t,k} \triangleq \nabla f_k(\boldsymbol{\theta}_{t})$ and global gradient as $\bold{g}_{t} \triangleq \nabla f(\boldsymbol{\theta}_{t})$, i.e., 
\begin{align} 
\bold{g}_{t} = \frac{1}{|\mathcal{D}|} 
	\sum_{k \in \mathcal{S}_t}  |\mathcal{D}_k|~ \bold{g}_{t,k}. \label{sumgrad2}
\end{align}
Using this notation, we can summarize the vanilla quasi-Newton method as follows: in iteration $t$,
\begin{enumerate}
	\item The PS aggregates local gradients as per \eqref{sumgrad2}.
	\item It computes the difference
	\begin{align} 
		&\bold{y}_t = 
		\bold{g}_{t} - \bold{g}_{t-1}. \label{sumgrad1}
	\end{align}
\item It finds a quasi-Newton matrix by solving \eqref{eq:B} for $\bold{B}_t$.
\item It updates the global model parameter as $\boldsymbol{\theta}_{t+1} = \boldsymbol{\theta}_{t} - \eta_t \bold{B}_t^{-1} \bold{g}_{t}$, and sets $\bold{w}_{t+1} = \boldsymbol{\theta}_{t+1} - \boldsymbol{\theta}_{t}$.
\end{enumerate}

\section{GP-FL: Over-the-Air Aggregation} \label{sec:method}
During each FL communication round, two parameter exchanges occur. The first is a downlink transmission where the PS broadcasts the global model. Since only one set of parameters is transmitted, the communication cost is negligible and can be handled via encoded data over the broadcast channel \cite{amiri2020machine}. The other parameter exchange occurs over uplink channels, where clients send their local gradients to the PS. This is the main source of communication overhead in wireless FL: using conventional orthogonal multiple-access transmission, the time-frequency resources scale linearly with the number of clients \(K\). For large \(K\), this leads to substantial overhead. To address this, we use AirComp to perform gradient aggregation directly over the air \cite{amiri2020machine,elgabli2021harnessing,yang2020federated,zhu2019broadband,sery2020analog,liu2020privacy}. This approach completes the aggregation within a single coherent block, significantly reducing communication overhead.

\subsection{Uplink Channel Model} \label{sec:comm}
We consider a Gaussian fading multiple access channel with slow frequency-flat fading, where the coherence time exceeds $d$ symbol intervals. Clients synchronously transmit $d$-dimensional local gradient entries over $d$ consecutive intervals within a single channel coherence interval. The PS is assumed to have an array of $N$ antennas, while each client has a single antenna.

We denote the channel coefficient vector between client $k$ and the PS in round $t$ as $\bold{h}_{t,k} \in \mathbb{C}^N$. Assuming perfect channel state information (CSI) at both ends, clients can adjust their transmitted signals based on the channel coefficients.


\subsection{Model Aggregation via AirComp}


The AirComp-based model aggregation works as follows: client $k$ sends $\phi_k(\bold{g}_{t,k})$ for some pre-processing function $\phi_k(\cdot)$ simultaneously with all other clients over the channel. The PS then receives a superimposed version of these transmissions. Denoting the received signal by $\bold{r}_t$, the PS estimates the aggregated gradient as $\hat{\bold{g}}_t = \psi (\bold{r}_t)$ for some post-processing function
$\psi(\cdot)$. The pre- and post-processing functions serve as analog filters aiming to fulfill the transmission constraints, e.g, transmit power constraint, and minimize the estimation error.

Due to the diversity of $\bold{g}_{t,k}$ among devices, a universal pre- and post-processing function cannot guarantee the joint stationarity of the information-bearing symbols. i.e., local gradients. 
We hence opt for a data-and-CSI-aware design: prior to uplink transmission, client $k$ normalizes its local gradient as 
\begin{align} \label{eq:pre-proc}
\bold{s}_{t,k} = \frac{\bold{g}_{t,k}}{\| \bold{g}_{t,k}\|_2}.
\end{align}
As the normalized gradients are unit-norm, we can model them 
as jointly stationary processes. This means that $\mathbb{E}(\|\bold{s}_{t,k}[j]\|^2)={1}/{d}$ for entry $ j \in [d]$.
After normalization, client $k$ sends 
\begin{align} \label{eq:bs}
\bold{x}_{t,k}   = b_{t,k}  \bold{s}_{t,k} ,
\end{align}
over its uplink channel, where $b_{t,k} \in \mathbb{R}$ is a 
power scaling factor satisfying
\begin{align} \label{eq:power}
\mathbb{E}(\vert b_{t,k}  \bold{s}_{t,k} [j] \vert^2) = \frac{b_{t,k}^2}{d} \leq \mathsf{P_0}, 
\end{align}
with $\mathsf{P_0}$ denoting maximum transmit power of devices.

The PS receives the superimposed version of transmitted signals: let $\bold{r}_{t,j} \in \mathbb{C}^N$ denotes the received signal of the PS in the $j$-th symbol interval of iteration $t$. Using the channel model, we can write
\begin{align} \nonumber
\bold{r}_{t,j} &= \sum_{k \in \mathcal{S}_t}   \bold{h}_{t,k}   \bold{x}_{t,k} [j] + \bold{n}_{t,j} = \sum_{k \in \mathcal{S}_t}   \hat{\bold{h}}_{t,k} b_{t,k} \bold{g}_{t,k} [j] + \bold{n}_{t,j}, 
\end{align}
where $\hat{\bold{h}}_{t,k}\triangleq{\bold{h}_{t,k}}/{\|\bold{g}_{t,k} \|_2}$ is the \textit{effective} channel, and $\bold{n}_{t,j} \in \mathbb{C}^N$ is complex additive white Gaussian noise (AWGN) with mean zero and variance $\sigma^2$, i.e., $\bold{n}_{t,j} \sim \mathcal{N}(\boldsymbol{0},\sigma^2\bold{I})$. 

The PS uses the received signals $\bold{r}_{t,j}$ for $j\in[d]$ to determine an estimator of the global gradient $\mathbf{g}_t$. Let $\bold{d}_t[j]$ denote the estimator of $\mathbf{g}_t[j]$. Using linear post-processing, the PS determines $\bold{d}_t[j]$ from $\bold{r}_{t,j}$ as 
\begin{align} \label{eq:r}
\bold{d}_t[j] &= \frac{\bold{c}_t^{\mathsf{H}} \bold{r}_{t,j}}{\sqrt{\alpha_t}}\\
&= \frac{1}{\sqrt{\alpha_t}} \left( \bold{c}_t^{\mathsf{H}} \sum_{k \in \mathcal{S}_t} \hat{\bold{h}}_{t,k} b_{t,k} \bold{g}_{t,k} [j] + \bold{c}_t^{\mathsf{H}}\bold{n}_{t,j}\right),
\end{align}
for some linear receiver $\bold{c}_t \in \mathbb{C}^N$ and the power factor $\alpha_t \neq 0$. 


Then, the PS employs a zero-forcing approach to estimate $\bold{g}_t$. It first computes a linear receiver $\bold{c}_t$ using the CSI (details on computing $\bold{c}_t$ will be discussed later) and determines the power scaling factor $\alpha^{\text{ZF}}_t$ as: 
\begin{align} \label{eq:zf}
\alpha^{\text{ZF}}_t = \mathsf{P_0} d \min_{k \in \mathcal{S}_t} \frac{\| \bold{c}_t^{\mathsf{H}}\hat{\bold{h}}_{t,k}\|_2^2}{|\mathcal{D}_k|^2}.  
\end{align}
This way, the PS guarantees that all clients satisfy their transmit power constraint, and then it broadcasts $\alpha^{\text{ZF}}_t$ to the clients. Client $k$ upon receiving $\alpha^{\text{ZF}}_t$  determines its scaling as
\begin{align} \label{eq:bk}
b_{t,k} = \sqrt{\alpha^{\text{ZF}}_t} |\mathcal{D}_k| \frac{ \hat{\bold{h}}_{t,k}^{\mathsf{H}} \bold{c}_t }{\| \bold{c}_t^{\mathsf{H}}\hat{\bold{h}}_{t,k}\|^2_2},
\end{align}
which satisfies the transmit power constraint \eqref{eq:power}. 

The above coordination of clients realizes the desired superposition over the air in the absence of noise during uplink transmission. Substituting \eqref{eq:zf} and \eqref{eq:bk} in \eqref{eq:r}, it is concluded that the PS aggregates $\bold{d}_t[j] = \sum_{k \in \mathcal{S}_t} |\mathcal{D}_k| \bold{g}_{t,k}[j] + \frac{1}{\sqrt{\alpha^{\text{ZF}}_t}} \bold{c}_t^{\mathsf{H}}  \bold{n}_{t,j}$, which after scaling by ${1}/{|\mathcal{D}|}$ concludes the following
\begin{align}
\Tilde{\bold{g}}_{t} [j] 
&= \frac{1}{|\mathcal{D}|} \bold{d}_t[j] 
= \bold{g}_{t} [j] + \frac{1}{|\mathcal{D}| \sqrt{\alpha^{\text{ZF}}_t}} \bold{c}_t^{\mathsf{H}}  \bold{n}_{t,j}.\label{eq:receveidnoisy}
\end{align}
The estimator of the gradient in iteration $t$ is hence given by 
\begin{align} \label{eq:distortg}
\Tilde{\bold{g}}_{t} =  \bold{g}_{t} +  \Tilde{\bold{n}}_t,  
\end{align}
where, for notational simplicity, we defined the $d$-dimensional column vector $\Tilde{\bold{n}}_t$ as
\begin{align} \label{eq:noisedef}
\Tilde{\bold{n}}_t= \frac{1}{|\mathcal{D}| \sqrt{\alpha^{\text{ZF}}_t}} [  \bold{c}_t^{\mathsf{H}} \bold{n}_{t,1} , \ldots, \bold{c}_t^{\mathsf{H}} \bold{n}_{t,d} ]^{\mathsf{T}}.  
\end{align}

\begin{remark}
In general we can assume $\bold{g}_t$ to be complex, as we can represent every two entries of a complex gradient as a pair of in-phase and quadrature components.
\end{remark}

\begin{remark}
Note that $\mathbb{E}[\Tilde{\bold{n}}_t]=0$, and thus $\Tilde{\bold{g}}_{t}$ is an unbiased estimator of $\bold{g}_{t}$. As such, assuming that $\bold{g}_{t}$ itself is an unbiased estimator of the true gradient, $\Tilde{\bold{g}}_{t}$ is an unbiased estimator of the true gradient, as well.   
\end{remark}


\subsection{Receiver Design for Model Aggregation} \label{sec:receive}
The variance of $\Tilde{\bold{n}}_t$ which specifies the variance of the gradient estimator is given by $\Tilde{\sigma}^2 (\bold{c}_t) = \frac{\sigma^2 \Vert \bold{c}_t \Vert^2} {|\mathcal{D}|^2 {\alpha^{\text{ZF}}_t}}$. Substituting \eqref{eq:zf} in this term, we have
\begin{align}
    \Tilde{\sigma}^2 (\bold{c}_t) = \frac{\sigma^2}{|\mathcal{D}|^2 \mathsf{P_0} d} \max_{k \in \mathcal{S}_t} \frac{|\mathcal{D}_k|^2 \Vert \bold{c}_t \Vert^2}{\| \bold{c}_t^{\mathsf{H}}\hat{\bold{h}}_{t,k}\|_2^2}.
\end{align}
Thus, for a given set of selected devices $\mathcal{S}_t$, the optimal receiver that minimizes the estimation variance is given by a solution to the following optimization problem:
\begin{align} \label{opt1}
\min_{\bold{c}_t}  \max_{k \in \mathcal{S}_t} \frac{|\mathcal{D}_k|^2 \Vert \bold{c}_t \Vert^2}{\| \bold{c}_t^{\mathsf{H}}\hat{\bold{h}}_{t,k}\|_2^2}  . 
\end{align}

Using a similar analysis as in \cite{8364613}, the optimization problem in \eqref{opt1} can be reformulated as
\begin{align}
    \min_{\bold{c}_t} \Vert \bold{c}_t \Vert^2 \qquad \text{s.t.} \quad    \| \bold{c}_t^{\mathsf{H}}\hat{\bold{h}}_{t,k}\|_2^2  \geq |\mathcal{D}_k|^2, \quad \forall k \in \mathcal{S}_t
\end{align}
which is a quadratically constrained quadratic programming problem that is computationally challenging. However, it can be transformed into a difference-of-convex-function (DC) program, as shown in Appendix \ref{app:DC}.

\begin{remark}
The receiver $\bold{c}_t$ is updated only once per channel coherence time, and hence it is realistic to assume that we coordinate the network after each update of $\mathbf{c}_t$, as it occurs with the same rate as for the CSI update.    
\end{remark}

\subsection{Device Selection Algorithm} \label{sec:select}
The device selection is a combinatorial problem, which lies int the set of nondeterministic polynomial (NP) hard problems. We therefore invoke the sub-optimal approach developed in \cite{9810113} to approximate its solution via Gibbs sampling (GS) method. The key idea in the GS method is to iteratively sample a device set from the neighboring devices based on a suitable distribution. This iterative process allows the selected devices to gradually converge towards an optimal set. We omit the details of the algorithm due to lack of space and refer the interested readers to \cite{9810113} for more details.

\section{GP-FL: Hessian Estimation} \label{sec:gaus}
In the quasi-Newton approach, the PS estimates the Newton direction via $\bold{B}_t$ computed from the subsequent gradient estimators $\bold{g}_{t}$ and $\bold{g}_{t-1}$ according to \eqref{eq:B}. With over-the-air aggregation, the gradient estimators, i.e., $\Tilde{\bold{g}}_{t}$ and $\Tilde{\bold{g}}_{t-1}$, are further distorted by effective channel noise process resulting in higher error variances. As demonstrated in our numerical investigations in Section \ref{sec:exp}, the direct derivation of $\bold{B}_t$ from the aggregated gradients $\Tilde{\bold{g}}_{t}$ and $\Tilde{\bold{g}}_{t-1}$ can significantly degrade the training performance as compared to the case with gradient estimators collected via noise-free aggregation. This observation can be intuitively explained as follows: the variance introduced by over-the-air aggregation leads (through the nonlinear procedure of solving \eqref{eq:B}) to an estimator whose estimate of Newton's direction is biased. The higher~the~aggregation variance is, the more this directional bias will be. 


Unlike the primal estimation error in the mini-batch gradients, i.e., $\bold{g}_{t}$, whose statistics is unknown, the computation error introduced by the over-the-air aggregation is statistically known to us. This knowledge can be potentially used to suppress the impact of computation error and extract a better estimator of the Hessian matrix. The key challenge in this respect is however that the Hessian estimator is related to the noisy aggregations through a nonlinear transform, i.e., the solution to \eqref{eq:B}. To overcome this challenge we follow the model-based approach suggested in various lines of work in the literature, e.g., \cite{wills2021stochastic}, to compute a better estimator of the Hessian matrix. 

\subsection{Model-based Hessian Estimation}
In model-based estimation approach, the stochastic nature of the solution to \eqref{eq:B} is described through a stochastic process, which approximates the original process. In the sequel, we adapt the Gaussian model for this problem. This choice is particularly useful as Gaussian processes offer non-parametric probabilistic models to capture complexities of nonlinear functions \cite{schulz2018tutorial}.\footnote{Modeling Hessians as Gaussian processes was proposed in \cite{wills2021stochastic} for use in noisy settings but has not yet been applied to tasks like FL.}

To understand the idea of probabilistic modeling of Hessian, let us look more precisely into the noisy quasi-Newton matrix computed from the over-the-air aggregations: 
at round $t$, the PS uses its latest two aggregations to compute 
\begin{align} \label{eq:nosiyy}
\Tilde{\bold{y}}_t =   \Tilde{\bold{g}}_{t} - \Tilde{\bold{g}}_{t-1},    
\end{align}
The quasi-Newton matrix $\Tilde{\bold{B}}_t$, which is an estimator of Hessian, can then be computed from $\Tilde{\bold{y}}_t$ by solving
\begin{align} \label{eq:noisyB}
\Tilde{\bold{B}}_t \bold{w}_{t} =  \Tilde{\bold{y}}_t.   
\end{align}
With ideal links, \(\Tilde{\bold{y}}_t = \bold{y}_t\), and using a deterministic scheme like BFGS, the solution to \eqref{eq:noisyB} matches that of \eqref{eq:B}. However, with over-the-air computation, deterministic methods yield a Hessian estimate with potentially higher bias and variance compared to the standard case with perfect aggregation\footnote{This is demonstrated in the numerical results.}.  To address this, we deviate from the classical approach of using deterministic schemes to solve \eqref{eq:noisyB}. Instead, we model \(\Tilde{\bold{B}}_t\) and \(\Tilde{\bold{y}}_t\) as jointly Gaussian random processes. We approximate their statistics from the observation sequence and use these to derive a more robust Hessian estimator, effectively mitigating the impact of aggregation noise.

\subsection{Quasi-Newton Matrix as Gaussian Process}
We start by modeling the \textit{prior}: let \(\Tilde{\bold{B}}_t\) be a Gaussian matrix with mean \(\mathbf{M}\) and covariance \(\bold{C}\). For \(i,j \in [d]\), entry \((i,j)\) of \(\Tilde{\bold{B}}_t\) is an independent Gaussian process with mean \(\mu_{i,j} = [\mathbf{M}]_{i,j}\) and variance \(\psi_{i,j}\) given by \(\bold{C}\). For simplicity, we focus on a single entry of \(\Tilde{\bold{B}}_t\), denoting it as \(\Tilde{b}_{t} \sim \mathcal{N}(\mu,\psi)\), dropping the index \((i,j)\). This does not affect the generality of the analysis, as entries of \(\Tilde{\bold{B}}_t\) are statistically independent and follow the same estimation procedure.

Our ultimate goal is to model $\Tilde{b}_t$ from the observations $\Tilde{\bold{y}}_t$ collected through time $t$. To this end, we collect the last $r$ gradient difference in a set $\mathcal{Y}_t$, i.e., 
\begin{align}
    \mathcal{Y}_t = \{ \Tilde{\bold{y}}_{t-r}, \dots , \Tilde{\bold{y}}_{t} \},
\end{align}
and model it jointly with entry $\Tilde{b}_t$ as a Gaussian processes $\mathcal{GP}$. More precisely, let us define $\bold{z} = [ \bold{o}_t^{\sf T}, \Tilde{b}_t ]^{\sf T}$, where $\bold{o}_t$ is defined as the concatenation of entries in $\mathcal{Y}_t$, i.e., 
\begin{align}
    \bold{o}_t = \left[ \Tilde{\bold{y}}_{t-r+1}^{\sf T}, \ldots, \Tilde{\bold{y}}_{t}^{\sf T} \right]^{\sf T}.
\end{align}
We model $\bold{z}_t$ as a Gaussian vector whose mean is $\boldsymbol{\xi}_t\in\mathbb{R}^{rd+1}$ and whose covariance matrix $\boldsymbol{\Sigma}_t \in\mathbb{R}^{(rd+1)\times (rd+1)}$. Considering this statistical model, our goal is to use the sample data collected in the last $r$ communication rounds, i.e., $\mathcal{Y}_t$, to find an estimate of the mean and covariance. We then use the assumed statistical model to find an alternative estimator for $\Tilde{b}_t$.

\subsection{Estimating Parameters of Gaussian Process}
To estimate the covariance matrix, we follow the conventional approach in the literature \cite{do2007gaussian}: we estimate the covariance of $\bold{z}_t$ from its samples using the kernel function $\kappa$. The choice of an appropriate kernel is based on assumptions such as smoothness and likely patterns that are expected in the data. Given the nature of data in this problem, i.e., the fact that $\Tilde{\bold{y}}_t$ is the gradient difference, one popular choice of the kernel is the radial basis function \cite{do2007gaussian}. 
\begin{definition}[Radial basis kernel]
For inputs $\bold{u} \in \mathbb{R}^{n}$ and $\bold{v} \in \mathbb{R}^{m}$, the radial basis function $\kappa$ computes an $n \times m $ matrix whose entry in the $i$-th row and $j$-th column is obtained as
\begin{align} \label{eq:kapp}
[\kappa(\bold{u},\bold{v})]_{i,j} = \exp \left( -\frac{\vert [\bold{u}]_i - [\bold{v}]_j \vert^2}{2\tau^2}    \right).  
\end{align}
\end{definition}
It is worth mentioning that using the radial basis kernel, the estimator of the covariance matrix computed from $\bold{z^0}$ sampled from $\bold{z}$, i.e., $\kappa(\bold{z^0},\bold{z^0})$, is a positive semi-definite matrix.

To use the above covariance estimator, we require a sample from $\bold{z}_t$. To this end, we use a deterministic solver to find a solution to \eqref{eq:noisyB} for the last $r$ communication rounds. This means that we find $\Tilde{\mathbf{B}}_i^0$ by solving \eqref{eq:noisyB} with $\Tilde{\bold{y}}_i^0$ for $i\in\{ t-r+1, \ldots,t \}$, where $\Tilde{\bold{y}}_i^0$ denotes to the sample gradient difference computed from the observed received signals.\footnote{Superscript $0$ is to distinguish random samples from stochastic processes.} Let $\Tilde{b}_i^0$ be an entry of $\Tilde{\mathbf{B}}_i^0$. We sample $\bold{z}$ as $\bold{z}^0 = [ \bold{o}_t^{0 \sf T}, \Tilde{b}_t^0 ]^{\sf T}$ with
\begin{align}
    \bold{o}_t^0 = \left[ \Tilde{\bold{y}}_{t-r+1}^{0\sf T}, \ldots, \Tilde{\bold{y}}_{t}^{0\sf T} \right]^{\sf T},
\end{align}
and estimate the covariance of $\bold{z}_t$ with the sample covariance matrix $\kappa(\bold{z}_t^0,\bold{z}_t^0)$ computed by the radial basis kernel.

We next compute an estimator of the mean $\boldsymbol{\xi}_t$ using a moving average: we estimate the mean of $\Tilde{b}_t$ by arithmetic average of $\Tilde{b}_{t-r+1}^0, \ldots, \Tilde{b}_{t}^0$ denoted by $\hat{\mathbb{E}} \{ \Tilde{b}_{t}^0 \}$, and the mean of $\bold{o}_i$ by the moving average of $\bold{o}_i^0$, i.e., we set the mean of $\Tilde{\bold{y}}_i$ to be the average of $\{ \Tilde{\bold{y}}_{t-r}^0, \dots , \Tilde{\bold{y}}_{i}^0 \}$ for $i\in\{ t-r+1, \ldots,t \}$. Let us denote this moving average by $\mu[\bold{o}_t^0]$. We can then approximate the random process $\bold{z}_t$ as
\begin{align}\label{eq:zt}
\bold{z}_t
\sim
\mathcal{N} 
\begin{pmatrix}
\begin{bmatrix}
\hat{\mathbb{E}} \{ \Tilde{b}_{t}^0 \}
\\
\mu[\bold{z}_t^0]
\end{bmatrix} 
, 
&
\begin{bmatrix}
\beta_t & \boldsymbol{\phi}_t^{\sf T}\\
\boldsymbol{\phi}_t & \bold{K}_t
\end{bmatrix} 
\end{pmatrix},
\end{align}
for scalar $\beta_t = \kappa (\Tilde{{b}}_{t}^0,\Tilde{{b}}_{t}^0)$, vector $\boldsymbol{\phi}_t = \kappa (\bold{o}_{t}^0,\Tilde{b}_{t}^0)$, and symmetric and positive semi-definite matrix $\bold{K}_t = \kappa (\bold{o}_{t}^0,\bold{o}_{t}^0)$.

\subsection{Estimating Hessian via Posterior Sampling}
The statistical model considered for the quasi-Newton matrix and the sample observations enables us to find an alternative estimator for the Hessian matrix. To this end, we derive the posterior distribution of $\Tilde{\bold{B}}_t$ in the following lemma.

\begin{lemma}[Posterior of quasi-Newton matrix]\label{lem:posterior}
 Consider the Gaussian model for $\bold{z}_t$ given in \eqref{eq:zt}. The entry $\Tilde{b}_t$ conditional to observation $\bold{o}_t$ is distributed Gaussian with mean $\zeta (\bold{o}_t)$ and variance $\psi (\bold{o}_t)$ that are given by
 \begin{subequations} \label{eq:sub}
\begin{align}
\zeta (\bold{o}_t) &= \hat{\mathbb{E}} \{ \Tilde{b}_{t}^0 \} - \bold{a}_t^{\sf T}
(\bold{o}_t - \mu[\bold{o}_t^0] ), \\ 
\psi (\bold{o}_t) &= \beta_t - \bold{a}_t^{\sf T} \boldsymbol{\phi}_t,
\end{align}
\end{subequations}
where $\bold{a}_t$ is the solution to $\bold{K}_t \bold{a}_t = \boldsymbol{\phi}_t$ which can be computed via the conjugate gradient algorithm.
\end{lemma}

\begin{proof}
    Since $\bold{z}_t$ is Gaussian, the posterior distribution of $\Tilde{b}_t$ is also Gaussian. After some standards derivations, the mean and variance of this conditional distribution are given by \cite{do2007gaussian}
     \begin{subequations} \label{eq:sub}
\begin{align}
\zeta (\bold{o}_t) &= \hat{\mathbb{E}} \{ \Tilde{b}_{t}^0 \} - \boldsymbol{\phi}_t^{\sf T} \bold{K}_t^{-1}
(\bold{o}_t - \mu[\bold{o}_t^0] ), \\ 
\psi (\bold{o}_t) &= \beta_t - \boldsymbol{\phi}_t^{\sf T} \bold{K}_t^{-1} \boldsymbol{\phi}_t.
\end{align}
By defining $\bold{a}_t$ as in the lemma, the proof is concluded.
\end{subequations}
\end{proof}

Using this posterior distribution, we can compute various estimators for $\Tilde{b}_t$. Noting that this estimator is used to estimate Newton's direction, we compute a simple unbiased estimator by sampling from the posterior distribution:\footnote{From the Bayesian viewpoint, one might consider setting the estimator to the posterior mean, as it describes the minimum mean squared error (MMSE) estimate. Although MMSE returns minimal variance, it is in general biased and hence can lead to directional misalignment.} for every entry of $\Tilde{\bold{B}}_t$, we compute the posterior mean $\zeta (\bold{o}_t^0)$ and variance $\psi (\bold{o}_t^0)$ using the sample observation $\bold{o}_t^0$. Let $\hat{b}_{i,j,t}$ be the posterior sample for entry $(i,j)$. We build the Hessian estimator $\hat{\bold{B}}_t$ from entries $\hat{b}_{i,j,t}$ and compute 
\begin{align} \label{eq:noisydir}
\Tilde{\mathfrak{d}}_t= -\hat{\bold{B}}^{-1}_t  \Tilde{\bold{g}}_{t}. 
\end{align}
The global model is then updated as $\boldsymbol{\theta}_{t+1} = \boldsymbol{\theta}_{t} + \eta_t \Tilde{\mathfrak{d}}_t$.

Although validated numerically, it is insightful to explain why \eqref{eq:noisydir} provides a better estimate of the true Newton direction compared to deterministic approaches like BFGS. Classical deterministic methods rely only on the latest two gradient samples, which can result in highly inaccurate curvature estimates under noisy aggregation. In contrast, the proposed scheme uses information from the last $r+1$ gradient samples to estimate the loss curvature. This allows it to effectively mitigate the impact of aggregation noise and compute a more robust curvature estimate.

\begin{remark}
From Lemma~\ref{lem:posterior}, computing the posterior involves solving an inverse problem to find $\bold{a}_t$, which may add computational cost compared to deterministic approaches. However, this inverse problem can be efficiently solved with reduced complexity using the conjugate gradient algorithm \cite{hestenes1952methods}.
\end{remark}

\section{Algorithm and Convergence Analysis} \label{sec:algconv}
The derivations in Sections~\ref{sec:method} and \ref{sec:gaus} enable the realization of a second-order FL framework over the air, referred to as \textit{Gaussian Process FL (GP-FL)}. Compared to classical second-order FL, GP-FL introduces two key aspects:
\begin{inparaenum}
    \item[($i$)] it uses analog function computation to aggregate local models directly over the air, and
    \item[($ii$)] it invokes the Gaussian model to compute a more robust estimation for Hessian from a finite window of recent aggregated gradients.
\end{inparaenum}
The final algorithm is summarized in Algorithm~\ref{alg}, outlining the framework with device scheduling. At communication round $t$, the PS selects a set of participating clients $\mathcal{S}_t$ from the $K$ available ones using a scheduling scheme, specifically the GS method from \cite{9810113}. The PS shares $\boldsymbol{\theta}_t$ with the devices. Upon receiving the global model, the devices compute local gradients and transmit them over their uplink channels. The PS aggregates the global gradient over the air and estimates the Hessian matrix using the method from Section~\ref{sec:gaus}, which is then used to update the model via \eqref{eq:noisydir}.

\begin{remark}
We assume the PS has high computational capabilities, rendering the complexity of calculating a direction via \eqref{eq:noisydir} negligible. Consequently, the overall running times for GP-FL and conventional second-order FL are nearly identical, as confirmed by our numerical experiments.
\end{remark}

\begin{algorithm}[t]
\caption{GP-FL} \label{alg}
{\bfseries Input:} Number of global epochs $T$, global learning rate ${\eta}_t$, sindow size $r$, local datasets $\{\mathcal{D}_k\}_{k \in K}$.\\
\For{$t=0,1,\dots,T-1$}{
PS randomly selects a subset of devices $\mathcal{S}_t$ and sends $\boldsymbol{\theta}_{t}$ to them. \\
    \For{device $k \in \mathcal{S}_t$ in parallel}{
        Compute local gradient as $ \bold{g}_{t,k} = \frac{1}{|\mathcal{D}_k|} \sum_{(\bold{u}, v) \in \mathcal{D}_k}  \ell (\boldsymbol{\theta}_t, \bold{u}, v )$.\\
        Normalize $\bold{g}_{t,k}$ according to \eqref{eq:pre-proc}, and find $\bold{x}_{t,k}$ based on \eqref{eq:bs}. \\
        Send $\bold{x}_{t,k}$ to the PS through wireless channel.
    }

    PS calculates $\Tilde{\bold{g}}_{t}$ as per \eqref{eq:receveidnoisy}.\\
    PS finds the updating direction as per \eqref{eq:noisydir}. \\ 
    PS updates the global model as  $\boldsymbol{\theta}_{t+1} = \boldsymbol{\theta}_{t} + \eta_t \Tilde{\mathfrak{d}}_t$. }
\end{algorithm}

\subsection{Scheduling the Learning Rate} \label{sec:lr}
The remaining of this section provides convergence analysis for GP-FL under a set of regularity assumptions, which are commonly considered in the literature \cite{amiri2020federated}. We start the analysis by stating the first assumption.

\begin{assumption}[Smoothness and convexity]
\label{asmp:1}
The global loss function is twice continuously differentiable, $L$-Lipschitz gradient ($L$-smooth) and $\lambda$-strongly convex. As such, we have
\begin{align} \label{eq:ass1}
\lambda \bold{I} \preceq  \nabla^2 f(\boldsymbol{\theta})  \preceq L   \bold{I}.
\end{align}
The strong convexity of the global loss function implies that there exists a unique optimal model parameter, which we denote by $\boldsymbol{\theta}^{\star}$ hereafter in the paper.   
\end{assumption}

For global convergence of GP-FL, as a stochastic estimator of the Newton method, under Assumption~\ref{asmp:1} the magnitude of model update should be controlled. The classical approach for update control is to resort backtracking line search methods. In this work, we deviate from this classical scheme and control the update magnitude by scheduling the global learning rate. To this end, we invoke the results of \cite{polyak2020new}, which proposes an adaptive learning rate to facilitate the global convergence of Newton-type methods. Casting GP-FL on the proposed scheme in \cite{polyak2020new}, it is shown that for convergence guarantee of GP-FL to the minimizer of an $L$-smooth and $\lambda$-strongly convex loss, it is sufficient to schedule the learning rate as
\begin{align} \label{eq:eta}
\eta_t= \min \{ 1 , \frac{\lambda^2}{L \| \bold{g}_t\|}\}.    
\end{align}
We hence consider this scheduling throughout the analysis.

\subsection{Convergence Results}
To present the convergence guarantee for GP-FL, we first need to define the concept of  \textit{$\delta$-approximate} matrix. 
\begin{definition}[$\delta$-approximate]
Matrix $\hat{\bold{H}}$ is said to be a $\delta$-approximate of $\bold{H}$, if the following inequality holds
\begin{align} \label{def:lambda}
\| \hat{\bold{H}} - \bold{H} \| \leq \delta \| \bold{H} \|.     
\end{align}
\end{definition}




Using this definition, we can now present our first convergence result: Theorem~\ref{eq:thconv} gives an upper-bound on the gap between the converging and optimal models.

\begin{theorem} \label{eq:thconv}
Let Assumption 1 hold. Moreover, let the inverse of sample quasi-Hessian matrix drawn from the posterior in Lemma~\ref{lem:posterior} in round $t$ be a $\delta_t$-approximate of the inverse Hessian matrix, and define $\delta = \max_t \delta_t$. Then, the distance between the global model updated in round $t$, i.e., $\boldsymbol{\theta}_t$ and the optimal model $\boldsymbol{\theta}^\star$ is bounded from above as
\begin{align} \label{eq:th11}
&\mathbb{E} \big[ \|\boldsymbol{\theta}_{t}-\boldsymbol{\theta}^{\star} \|\big]\leq
\mu^t \mathbb{E} \big[ \big\| \boldsymbol{\theta}_{0}-\boldsymbol{\theta}^{\star} \big\| \big]  + C_t,
\end{align}
where $\mu = L\delta/\lambda$ and $C_t$ is given by
\begin{align} \label{eq:th11}
C_t = \frac{\mu^t-1}{\mu-1}
\begin{cases}
C_0 &t\leq t_0 \\
C_1 &t > t_0
\end{cases},
\end{align}
for $C_0$ and $C_1$ being
\begin{subequations} \label{eq:th_cond}
\begin{align}
C_0 &= 
\frac{\lambda}{L} (t_0 - t + \frac{2\gamma}{1-\gamma}) + \frac{\delta+1}{\lambda} \frac{\sigma_{n} \big\| \bold{c}_t \big\|}{|\mathcal{D}| \sqrt{\alpha^{\mathrm{ZF}}_t}} \\
C_1 &= 
\frac{2\lambda \gamma ^{2^{t-t_0}}}{L (1-\gamma ^{2^{t-t_0}})}+ \frac{\delta+1}{\lambda} \frac{\sigma_{n} \big\| \bold{c}_t \big\|}{|\mathcal{D}| \sqrt{\alpha^{\mathrm{ZF}}_t}} 
\end{align}
\end{subequations}
with $t_0 = \max \Bigl\{ 0 , \bigl \lceil \frac{2L}{\lambda^2 \| \bold{g}_0 \|} \bigr \rceil -2 \Bigr\}$
 and $\gamma = \frac{L}{2\lambda^2 } \| \bold{g}_0 \| - \frac{t_0}{4}$.
\end{theorem}
\begin{proof}
The proof is given in Appendix~\ref{app:proof}.    
\end{proof}
Theorem \eqref{eq:thconv} implies that the GP-FL algorithm exhibits a \textit{linear-quadratic} convergence rate. In fact, the quadratic term in \eqref{eq:th11} is exactly the same as the one reported in \cite{polyak2020new} for Newton-type methods. Nevertheless, our result has an extra linear term, as a result of model-based estimation of the Hessian matrix and the AirComp aggregation error. 

We next present Corollaries~\ref{corr:th1} and \ref{corr:th2}, which characterize the scenarios under which the GP-FL algorithm converges quadratically and linearly. 

\begin{corollary} \label{corr:th1}
Let the initial model be in a bounded distance of the optimal model as
\begin{align}
\mathbb{E} \big[ \big\| \boldsymbol{\theta}_{0}-\boldsymbol{\theta}^{\star} \big\| \big] <  \frac{\mu^t-1}{\mu^t(\mu-1)} C_1 .  
\end{align}
Then, the updated model in round $t$ satisfies 
\begin{align}
\mathbb{E} \big[ \|\boldsymbol{\theta}_{t}-\boldsymbol{\theta}^{\star} \|\big] \leq
\frac{2(\mu^t - 1)}{\mu - 1} C_1.
\end{align}
and lies within the $\varepsilon$-neighborhood of optimal model after $T_\varepsilon$ rounds, i.e., $\mathbb{E} \big[ \big\| \boldsymbol{\theta}_{T_{\varepsilon}}-\boldsymbol{\theta}^{\star} \big\| \big] \leq \varepsilon$, if $\mu<1$ and
\begin{align}
T_{\varepsilon} = \mathcal{O} \left( \log \log \frac{1}{\varepsilon}\right). 
\end{align}
which is also called super-linear convergence rate. 
\end{corollary}
\begin{proof}
    See Appendix~\ref{app:corproof}.
\end{proof}
Corollary~\ref{corr:th1} presents an intuitive result: when GP-FL is initiated in a \textit{close enough} vicinity of the optimal model, it can \textit{quadratically} converge to the optimal model. We next consider the case with linear convergence.

\begin{corollary} \label{corr:th2}
Let the initial model satisfies
\begin{align}
\mathbb{E} \big[ \big\| \boldsymbol{\theta}_{0}-\boldsymbol{\theta}^{\star} \big\| \big] \geq  \frac{\mu^t-1}{\mu^t(\mu-1)} C_1 .  
\end{align}
Then, the updated model in round $t$ satisfies
\begin{align}
\mathbb{E} \big[ \|\boldsymbol{\theta}_{t}-\boldsymbol{\theta}^{\star} \|\big] \leq 2 \mu^t 
\mathbb{E} \big[ \|\boldsymbol{\theta}_{0}-\boldsymbol{\theta}^{\star} \|\big].   
\end{align}
and lies within the $\varepsilon$-neighborhood of optimal model after $T_\varepsilon$ rounds, i.e., $\mathbb{E} \big[ \big\| \boldsymbol{\theta}_{T_{\varepsilon}}-\boldsymbol{\theta}^{\star} \big\| \big] \leq \varepsilon$, if $\mu<1$ and
\begin{align}
T_{\varepsilon} = \mathcal{O} \left( \frac{1}{\log(\frac{1}{\mu})}\log \frac{1}{\varepsilon}\right)  .  
\end{align}
\end{corollary}
\begin{proof}
   See Appendix~\ref{app:corproof2}.
\end{proof}

It is worth mentioning that in classical distributed methods, the number of rounds required to achieve a desired precision $\varepsilon$, follows a linear convergence rate. Specifically, for vanilla federated averaging $T_{\varepsilon} = \mathcal{O} \left( \frac{L}{\lambda} \log \frac{1}{\varepsilon} \right)$. This underscores the superiority of GP-FL in terms of convergence rate.     

The convergence analysis can be extended to characterize the optimality gap, i.e., the difference between $f (\boldsymbol{\theta}_{t})$ and the minimal loss. Theorem~\ref{th:lossconv} gives an upper bound on the optimality gap at round $t$ whose proof is differed to Appendix \ref{app:thf}.

\begin{theorem} \label{th:lossconv}
Consider the assumptions in Theorem~\ref{eq:thconv}. Then, the optimality gap in round $t$ is bounded from above as
\begin{align}
& \mathbb{E} \big[  f (\boldsymbol{\theta}_{t}) \big] - f (\boldsymbol{\theta}^{\star})  \leq \frac{L}{2} \left( {\mu^{2t}} \mathbb{E} \big[ \big\| \boldsymbol{\theta}_{0}-\boldsymbol{\theta}^{\star} \big\|^2 \big] + C'_t \right),
\label{eq:boundfin3}
\end{align}
where $C'_t$ is defined as 
    \begin{align}
       C'_t = \frac{\mu^{2t}-1}{\mu^2-1} \begin{cases}
 C^{\prime}_0 &t\leq t_0 \\
C^{\prime}_1 &t > t_0 
\end{cases}
    \end{align}
for $\mu$ that is defined in Theorem~\ref{eq:thconv}, and 
\begin{subequations} 
\begin{align}
&C^{\prime}_0= \frac{\lambda^2}{L^2} (t_0 - t + \frac{2\gamma}{1-\gamma})^2 + \left(\frac{\delta+1}{\lambda} \right)^2 \frac{\sigma^2_{n} \big\| \bold{c}_t \bold{c}^{\mathsf{H}}_t\big\|}{|\mathcal{D}|^2 \alpha^{\text{ZF}}_t}, \\
 &C^{\prime}_1= \left(\frac{2\lambda \gamma ^{2^{t-t_0}}}{L (1-\gamma ^{2^{t-t_0}})} \right)^2+  \left( \frac{\delta+1}{\lambda} \right)^2 \frac{\sigma^2_{n} \big\| \bold{c}_t \bold{c}^{\mathsf{H}}_t\big\|}{|\mathcal{D}|^2 \alpha^{\text{ZF}}_t}.
\end{align}
\end{subequations}

\end{theorem}

\section{Simulation Results} \label{sec:exp}
We validate GP-FL by numerical experiments. 
Throughout the experiments, we elaborate the effectiveness of our proposed scheme by comparing it with several \textit{benchmarks}, namely the following first-order and second-order FL algorithms:
\paragraph*{AirComp-FedAvg (First-order)} Our first benchmark is AirComp-FedAvg proposed in \cite{yang2020federated}. The algorithm schedules devices using difference-of-convex programming.
\paragraph*{GIANT (Second-order)}  We consider GIANT algorithm \cite{wang2018giant} when realized directly over the air via AirComp. We note that GIANT requires an additional aggregation of local gradients, resulting in two communication rounds per iteration. The communication model for this gradient aggregation follows the same procedure as we discussed in Section \ref{sec:comm}.
\paragraph*{DANE (Second-order)} We deploy AirComp to realize the DANE algorithm proposed in \cite{shamir2014communication}. Similar to GIANT, this algorithm requires two aggregations per communication round. 
\paragraph*{Sec-Order (Second-order)} We consider the Sec-Order algorithm \cite{9810113}, where the clients locally find a Newton direction and send them to the PS via AirComp. Note that this approach loads more computations on local devices.
\paragraph*{BFGS} The PS uses BFGS to find quasi-Newton direction based on the AirComp-based noisy gradients.

We follow the simulation setup in \cite{10091800} and consider channels with 200 distinct noise standard deviation values, namely $\{0.005, 0.0100, \dots, 1\}$. For each value, we generate a corresponding channel. Then, based on the number of devices $ K $, we sample $ K $ values from this noise set. The server is equipped with $N = 5$ antennas. For a fair comparison, we apply the same receiver beamforming method described in Section \ref{sec:receive} and the same client selection method discussed in Section \ref{sec:select} across all benchmark methods, including GP-FL. For all experiments, we use the SGD optimizer on local devices with a learning rate of $ \eta = 0.1 $, momentum set to 0.9, and weight decay of 0.0005. The batch size is set to 64. The hyper-parameters of the benchmark methods are aligned with those reported in their respective original papers. Additionally, we set the observation window width in GP-FL to $ r = 20 $. Other hyper-parameters are specified for each experiment separately.

\begin{figure*}[t]
	\centering
     \subfloat[w8a]{\includegraphics[width=0.25\textwidth]{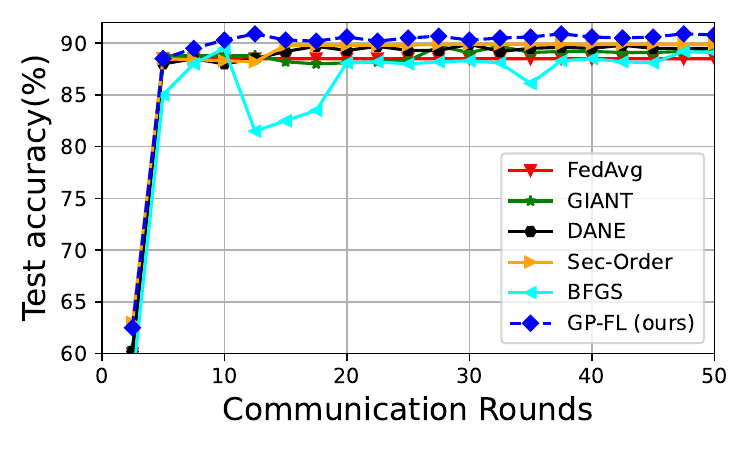}\label{fig:w8a}}
     \subfloat[a9a]{\includegraphics[width=0.25\textwidth]{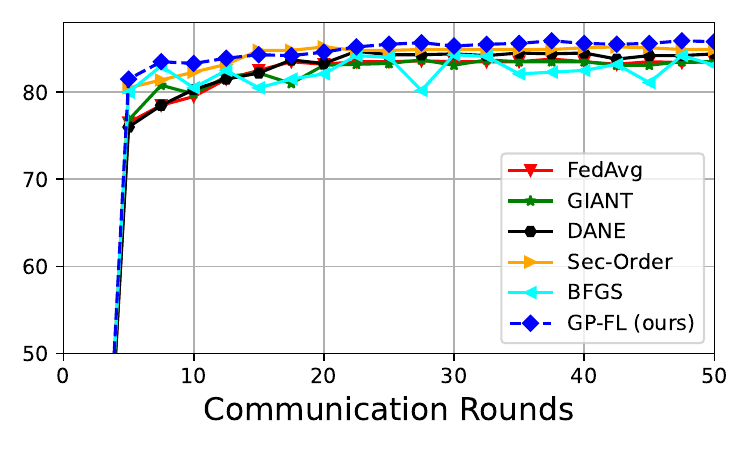}\label{fig:a9a}}
     \subfloat[phishing]{\includegraphics[width=0.25\textwidth]{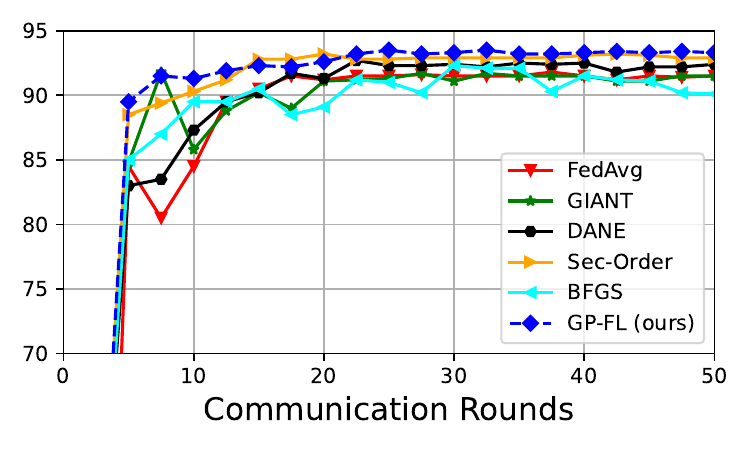}\label{fig:phishing}}
    \subfloat[Fashion-MNIST]{\includegraphics[width=0.25\textwidth]{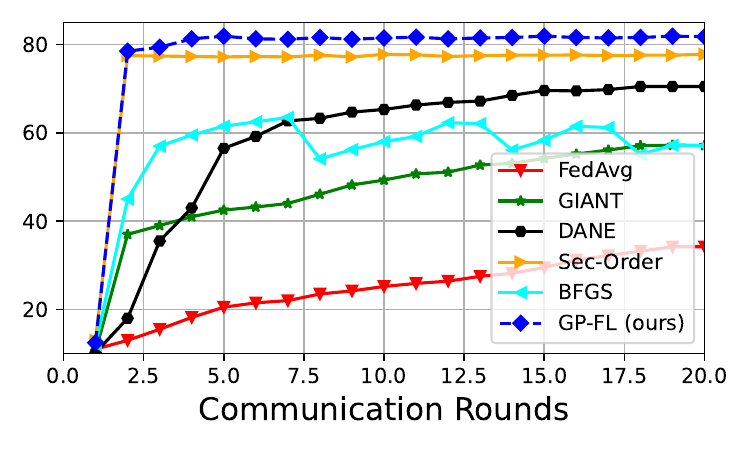}\label{fig:fmnist}}\\ \vskip -0.15in
	\subfloat[CIFAR-10, Setup I]{\includegraphics[width=0.25\textwidth]{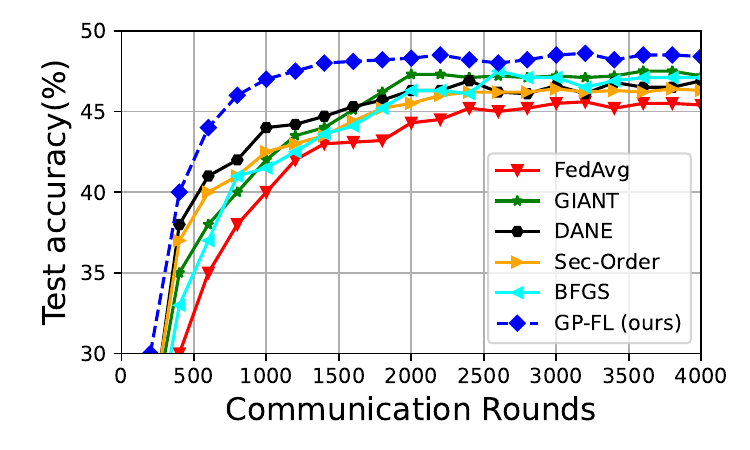}\label{fig:cifar10-1}}
	\subfloat[CIFAR-10, Setup II]{\includegraphics[width=0.25\textwidth]{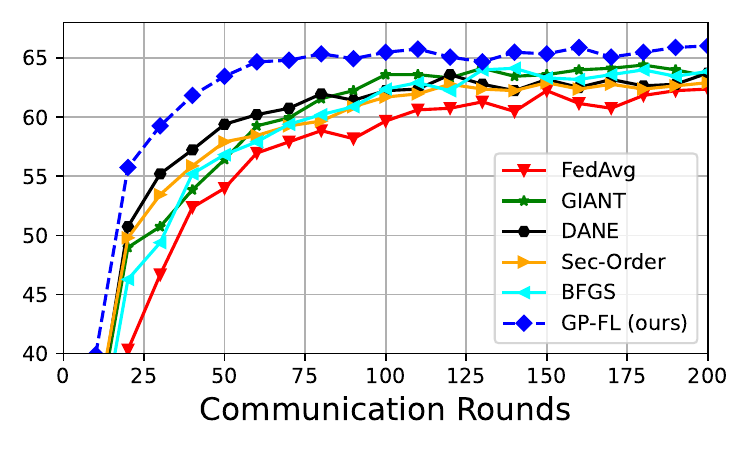}\label{fig:cifar10-2}} 
 \subfloat[CIFAR-100, Setup I]{\includegraphics[width=0.25\textwidth]{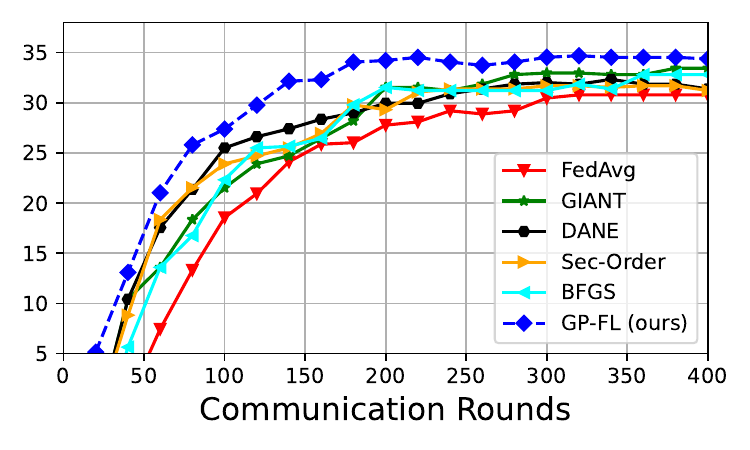}\label{fig:cifar100-1}}
	\subfloat[CIFAR-100, Setup II]{\includegraphics[width=0.25\textwidth]{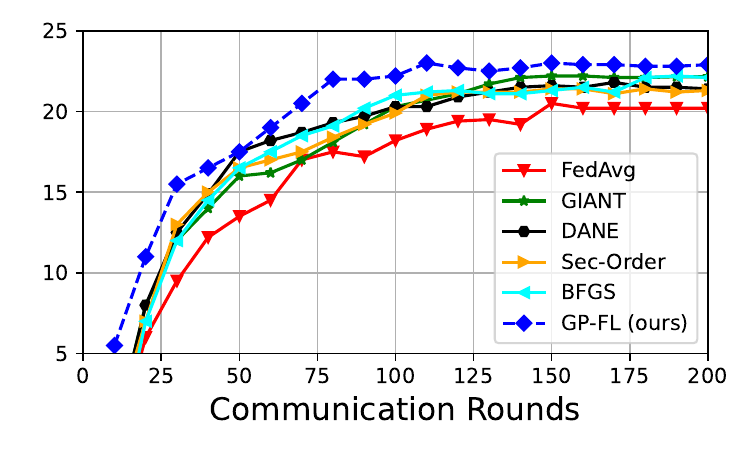}\label{fig:cifar100-2}} 
    \vskip -0.05in
    \caption{Test accuracy vs. communication rounds for GP-FL and baseline methods across different datasets.}
	\label{fig:improve}
	\vspace{-0.5cm}
\end{figure*}

\subsection{LIBSVM Dataset} \label{sec:LIBSVM}
First, we show results for a binary classification task using logistic regression and three datasets from the LIBSVM library \cite{chang2011libsvm}: a9a, w8a, and phishing. The data samples are uniformly distributed across $K = 20$ clients, all of whom participate in 50 communication rounds. We use the logistic regression model, a type of generalized linear model, as described in \cite{hosmer2013applied}. The results for a9a, w8a, and phishing datasets are depicted in \cref{fig:w8a,fig:a9a,fig:phishing}. As observed, GP-FL achieves higher classification accuracy than benchmark methods.

\subsection{Fashion-MNIST Dataset} \label{sec:LIBSVM}
We next use a fully-connected neural network with two hidden layers and perform 300 communication rounds, with all $ K = 10 $ clients participating in each round. The results are shown in \cref{fig:fmnist}. As observed, GP-FL not only achieves higher classification accuracy than the benchmark methods but also reaches 80\% accuracy within just five rounds.

\subsection{CIFAR-10 Dataset (Non-iid Settings)} \label{sec:CIFAR-10}
We evaluate CIFAR-10 classification in two setups:
\begin{itemize}
    \item \textit{Setup I}: Following \cite{hamidifair,10381881}, we organize the dataset by class and divide it into 200 shards. Each client randomly selects two shards without replacement. We use a feedforward neural network with two hidden layers for the FL task, with $ K=100 $ clients, 4000 communication rounds, and a 10\% client sampling rate per round.
    
    \item \textit{Setup II}: We distribute the dataset among them using Dirichlet allocation  with $\beta=0.5$. Using ResNet-18 with Group Normalization, we conduct 200 communication rounds, wherein all clients participate. We set $K=10$.
\end{itemize}

 The test accuracy curves for our method and the benchmark methods are depicted in \cref{fig:cifar10-1,fig:cifar10-2}. As illustrated, compared to the benchmark methods, GP-FL (i) achieves a higher test accuracy, and (ii) demonstrates a faster rate of convergence. For example, in \cref{fig:cifar10-1}, GP-FL reaches 45\% test accuracy in approximately 700 rounds, whereas the best-performing second-order benchmark methods require around 1600 rounds to reach the same accuracy.

\subsection{CIFAR-100 Dataset (Non-iid Settings)}\label{sec:cifar100}
CIFAR-100 shares the same sample size as CIFAR-10, yet it encompasses a broader diversity with 100 distinct classes. We next train the ResNet-18 model for this classification. Through training, we use Group Normalization and let all clients participate in each communication round. For this experiment, we consider the following two setups:
\begin{itemize}
    \item \textit{Setup I}: We set $K=10$ and $\beta=0.5$ for Dirichlet allocation, and use 400 communication rounds.
    \item \textit{Setup II}: We set $K=50$ and $\beta=0.05$ for Dirichlet allocation, and use 200 communication rounds.
\end{itemize}

The test accuracy curve for our method and the benchmark methods are depicted in \cref{fig:cifar100-1,fig:cifar100-2}. Similar conclusions can be drawn from these figures as those from the CIFAR-10 dataset. For example, in \cref{fig:cifar100-1}, GP-FL reaches 30\% test accuracy in just 120 rounds, whereas the best benchmark method requires 180 rounds to achieve the same performance.

\subsection{Comments on Required Observation Window} \label{sub:r}
From an implementation perspective, GP-FL requires the PS to select an efficient observation window $r$. We empirically examine how different $r$ values affect GP-FL's performance by repeating the experiment in Setup I with CIFAR-10 (see Subsection \ref{sec:CIFAR-10}) for $r = \{0, 5, 10, 15, 20, 50, 100\}$. Setting $r = 0$ corresponds to the standard quasi-Newton algorithm. The results, summarized in \cref{tab:r}, show no improvement beyond $r = 20$, with accuracy dropping at $r = 100$. This is because recent gradients carry the most relevant curvature information. Thus, the algorithm can be efficiently implemented with a reasonably sized observation window.

\vskip -0.1in
\begin{table}[!h] 
\caption{Impacts of hyper-parameter $r$: \normalfont the simulation setup is the same as Setup I with CIFAR-10.}
\vskip -0.2in
\label{tab:r}
\begin{center}
\resizebox{.95\columnwidth}{!}{%
\begin{tabular}{c|c|c|c|c|c|c|c}
\toprule
$r$ & 0 & 5 & 10 & 15 & 20 & 50 & 100\\ \hline

Average Accuracy (\%) & 45.24 & 46.75 & 47.88 & 48.41 & 48.59 & 48.61 & 48.55 \\
\bottomrule
\end{tabular}}
\end{center}
\vskip -0.2in
\end{table}

\section{Conclusion} \label{sec:conclusion}
While second-order methods are favored in FL for their fast convergence, the communication cost of sharing local Hessian matrices poses a challenge. Attempts to approximate these matrices with first-order information have limited effectiveness in real-world scenarios due to noisy updates from imperfect wireless channels. In this paper, we propose a novel second-order FL algorithm tailored for wireless channels, termed Gaussian process-based Hessian modeling for FL (GP-FL). The algorithm uses a non-parametric approach to estimate the global Hessian matrix from noisy local gradients received over uplink channels. Our analysis demonstrates that GP-FL achieves a linear-quadratic convergence rate. Extensive numerical experiments validate the proposed method's efficiency. Extending GP-FL to other distributed learning approaches is a natural direction in this area.

\appendices 

\section{Solving Optimization \eqref{opt1}} \label{app:DC}
To solve the joint optimization in \eqref{opt1}, we begin by transforming it into a low-rank optimization problem. Specifically, let $\bold{C} = \bold{c}_t \bold{c}_t^{\mathsf{H}}$, where $\text{rank}(\bold{C}) = 1$, and let $\bold{H}_k = \hat{\bold{h}}_{t,k} \hat{\bold{h}}_{t,k}^{\mathsf{H}}$. This allows us to reformulate \eqref{opt1} as 
\begin{align}
&\min_{\bold{C}}~ \text{Tr}(\bold{C})  \\ \nonumber
& \text{s.t.}~ \bold{C} \succeq \bold{0},~ \text{rank}(\bold{C})=1,~ \text{Tr}(\bold{C}\bold{H}_k) \geq |\mathcal{D}_k|^2, \quad \forall k \in \mathcal{S}_t.
\end{align}
Effectively addressing the rank-one constraint is crucial for solving low-rank optimization problems. A common approach is semidefinite relaxation (SDR), which removes the constraint, reformulating the problem as semidefinite programming to approximate the solution. However, for larger problems, the rank-one constraint is often violated, requiring randomization methods to scale the solution. 

An alternative approach to enforce the rank-one constraint is to express it as $\text{Tr} (\bold{C}) - \| \bold{C} \|_2 = 0$ with $\text{Tr} (\bold{C}) > 0$. This reformulates the problem into a difference-of-convex-functions (DC) program, allowing for a more accurate solution by satisfying all constraints. Building on methods in \cite{9810113}, we introduce the following modified DC programming, incorporating the new constraint as a penalty term:
\begin{align}
&\min_{\bold{C}} ~\text{Tr}(\bold{C}) + \zeta (\text{Tr} (\bold{C}) - \| \bold{C}\|_2)\\ \nonumber
& \text{s.t.}~ \bold{C} \succeq \bold{0},~ \text{Tr}(\bold{C})>0,~ \text{Tr}(\bold{C}\bold{H}_k) \geq |\mathcal{D}_k|^2, \quad \forall k \in \mathcal{S}_t, 
\end{align}
where $\zeta$ is a regularizer. Although this remains a non-convex problem due to the concave term $-\|\bold{C} \|_2$, we can linearize $\|\bold{C} \|_2$ and transform it into the following convex iterative subproblem:
\begin{align}
&\min_{\bold{C}} ~(1+\zeta)\text{Tr}(\bold{C}) - \zeta \partial \| \bold{C}_j\|_2 \cdot \bold{C}\\ \nonumber
& \text{s.t.}~ \bold{C} \succeq \bold{0},~ \text{Tr}(\bold{C})>0,~ \text{Tr}(\bold{C}\bold{H}_k) \geq |\mathcal{D}_k|^2, \quad \forall k \in \mathcal{S}_t, 
\end{align}
where $\bold{C}_j$ is $j$-th iterative of $\bold{C}$, $\partial \| \bold{C}_j\|_2$ is the sub-gradient of $\| \bold{C}_j\|_2$. For a complete description of the iterative algorithm, please refer to Algorithm 2 in \cite{9810113}.

\section{Proof of Theorem~\ref{eq:thconv}} \label{app:proof}
Before starting the derivation, note that we frequently use the triangle inequality, which states that for two vectors $\bold{v}$ and $\bold{u}$, $\|\bold{v} \pm \bold{u}\| \leq \|\bold{v}\| + \|\bold{u}\|$. For clarity, we will omit explicitly mentioning its use when it is evident from the context.

We aim to bound the distance to the optimal model $\boldsymbol{\theta}^{\star}$ at round $t$, i.e., $\|\boldsymbol{\theta}_t - \boldsymbol{\theta}^{\star}\|$. This bound is derived by establishing a recursive relationship between the distances in consecutive communication rounds. To begin, we note that
$\boldsymbol{\theta}_{t+1} = \boldsymbol{\theta}_{t} + \eta_t \Tilde{\mathfrak{d}}_t$, where $\Tilde{\mathfrak{d}}_t= -\hat{\bold{B}}^{-1}_t  \Tilde{\bold{g}}_{t}$ as given in \eqref{eq:noisydir}. We hence have
\begin{align}
\|\boldsymbol{\theta}_{t+1}-\boldsymbol{\theta}^{\star} \|   &= \big\| \boldsymbol{\theta}_{t} + \eta_t \Tilde{\mathfrak{d}}_t - \boldsymbol{\theta}^{\star} \big\| \leq \mathcal{M}_1 + \mathcal{M}_2,
\label{eq:app1}
\end{align}
where we define $\mathcal{M}_1 = \big\| \boldsymbol{\theta}_{t} -  \eta_t  \bold{H}^{-1}_t   \bold{g}_{t} - \boldsymbol{\theta}^{\star} \big\|$, and
\begin{align}
\mathcal{M}_2 &= \big\|  \bold{H}^{-1}_t   \bold{g}_{t} +\Tilde{\mathfrak{d}}_t \big\| = \big\|  \bold{H}^{-1}_t   \bold{g}_{t} - \hat{\bold{B}}^{-1}_t  \Tilde{\bold{g}}_{t} \big\|.
\label{eq:app2}
\end{align}

The term $\mathcal{M}_1$ is upper-bounded using the results reported in \cite{polyak2020new}. In particular, defining $t_0$ and $\gamma$ as in Theorem~\ref{eq:thconv}, the result of \cite{polyak2020new} implies that
\begin{align} \label{eq:polyak}
\mathcal{M}_1 \leq \frac{\lambda}{L}
\begin{cases}
t_0 - t + \dfrac{2\gamma}{1-\gamma}, ~ &t\leq t_0 \\
\dfrac{2 \gamma^{2^{t-t_0}}}{1-\gamma ^{2^{t-t_0}}}, &t > t_0
\end{cases}.
\end{align}

In the next step, we find an upper bound on $\mathcal{M}_2$. To this end, we use the triangular inequality once again to write
\begin{subequations}
\begin{align} \label{eq:M2}
\mathcal{M}_2 &\leq  
\big\|  \bold{H}^{-1}_t   \bold{g}_{t} - \hat{\bold{B}}^{-1}_t  \bold{g}_{t}  \big\| 
+ 
\big\| \hat{\bold{B}}^{-1}_t  \bold{g}_{t} - \hat{\bold{B}}^{-1}_t  \Tilde{\bold{g}}_{t}  \big\|  \\ \label{eq:M2_2}
& \leq   \big\| \bold{H}^{-1}_t - \hat{\bold{B}}^{-1}_t \big\| \big\| \bold{g}_{t} \big\|   + \big\|  \hat{\bold{B}}^{-1}_t \big\|   \big\| \bold{g}_{t} - \Tilde{\bold{g}}_{t} \big\|.
\end{align}
\end{subequations}
Noting that $\hat{\bold{B}}^{-1}_t$ is $\delta_t$-approximator of $\mathbf{H}_t^{-1}$, we can write
\begin{align}
\big\| \bold{H}^{-1}_t - \hat{\bold{B}}^{-1}_t \big\| \leq \delta_t \big\| \bold{H}^{-1}_t \big\| \leq \delta \big\| \bold{H}^{-1}_t \big\|, 
\end{align}
with the latter inequality being concluded directly from the fact that $\delta = \max_t \delta_t$. We further note that $\lambda$-strong convexity of the loss yields $\big\| \bold{H}^{-1}_t \big\| \leq \frac{1}{\lambda}$. We hence conclude that
\begin{align} \label{eq:difinv}
\big\| \bold{H}^{-1}_t - \hat{\bold{B}}^{-1}_t \big\| \leq \frac{\delta}{\lambda},
\end{align}
and use the $L$-smoothness of the global loss to write 
\begin{align} \label{eq:Lsmooth}
\big\| \bold{g}_{t} \big\| \leq L \big\| \boldsymbol{\theta}_{t}-\boldsymbol{\theta}^{\star} \big\|.    
\end{align}
Using \eqref{eq:difinv} and \eqref{eq:Lsmooth}, the first term in \eqref{eq:M2_2} is straightforwardly bounded. To bound the second term, we have
\begin{subequations}
\begin{align} \label{eq:binvlim}
\big\|  \hat{\bold{B}}^{-1}_t  \big\| & \leq \big\|  \hat{\bold{B}}^{-1}_t - \bold{H}^{-1}_t \big\| + \big\| \bold{H}^{-1}_t \big\| \\
&\leq \frac{\delta}{\lambda} + \frac{1}{\lambda} = \frac{\delta+1}{\lambda}.
\end{align}
\end{subequations}
Using this bound, we can finally write
\begin{align} \label{eq:M2bound}
\mathcal{M}_2 \leq \mu \big\| \boldsymbol{\theta}_{t}-\boldsymbol{\theta}^{\star} \big\| + \frac{\delta+1}{\lambda} \big\|  \bold{g}_{t} - \Tilde{\bold{g}}_{t} \big\|,
\end{align}
with $\mu$ being defined as in Theorem~\ref{eq:thconv}.

The above bound leads to an \textit{instantaneous} recursive bound. We now recall that in the AirComp-aided aggregation we have $\Tilde{\bold{g}}_{t} - \bold{g}_{t} = \Tilde{\bold{n}}_t$. By taking expectation from both sides of \eqref{eq:M2bound} w.r.t. to the randomness in communication links, we obtain 
\begin{align} \label{eq:Em2}
\mathbb{E} \big[\mathcal{M}_2 \big] \leq    \mu \mathbb{E} \big[ \big\| \boldsymbol{\theta}_{t}-\boldsymbol{\theta}^{\star} \big\| \big]  + \frac{\delta+1}{\lambda} \mathbb{E} \big[ \big\| \Tilde{\bold{n}}_t \big\| \big]. 
\end{align}
Using \eqref{eq:noisedef}, we can further write 
\begin{subequations}
\begin{align} \label{eq:noiseboumd}
\mathbb{E} \big[ \big\| \Tilde{\bold{n}}_t \big\| \big] &=\frac{1}{|\mathcal{D}| \sqrt{\alpha^{\text{ZF}}_t}}  \mathbb{E} \big[ \big\| \bold{n}_{t,j}^{\mathsf{H}} \bold{c}_t \big\| \big] \\
& \leq \frac{1}{|\mathcal{D}| \sqrt{\alpha^{\text{ZF}}_t}}  \big\| \bold{c}_t \big\|  \mathbb{E} \big[ \big\|  \bold{n}_{t,j} \big\| \big] = \frac{\sigma_{n} \big\| \bold{c}_t \big\|}{|\mathcal{D}| \sqrt{\alpha^{\text{ZF}}_t}}.
\end{align}
\end{subequations}
We finally take an expectation from both sides of \eqref{eq:app1}, and use \eqref{eq:Em2} and \eqref{eq:noiseboumd} to write (note that $\eta_t \leq 1$ as per \eqref{eq:eta})
\begin{align} \label{eq:boundfin1}
&\mathbb{E} \big[ \|\boldsymbol{\theta}_{t}-\boldsymbol{\theta}^{\star} \|\big]\leq
\mu \mathbb{E} \big[ \big\| \boldsymbol{\theta}_{t}-\boldsymbol{\theta}^{\star} \big\| \big]  + A_t,
\end{align}
where $A_t$ is defined as
\begin{align} \label{eq:th11}
A_t = 
\begin{cases}
C_0 &t\leq t_0 \\
C_1 &t > t_0
\end{cases}
\end{align}
for $C_0$ and $C_1$ that were defined in Theorem~\ref{eq:thconv}. Applying \eqref{eq:boundfin1} recursively, the proof is concluded. %

\section{Proof of Corollary \ref{corr:th1}} \label{app:corproof}
From Theorem \eqref{eq:thconv}, it is concluded that if $\mathbb{E} \big[ \big\| \boldsymbol{\theta}_{0}-\boldsymbol{\theta}^{\star} \big\| \big] <  C_t$
\begin{align}
\mathbb{E} \big[ \|\boldsymbol{\theta}_{t}-\boldsymbol{\theta}^{\star} \|\big] &\leq  2C_t .
\end{align}
We aim to determine the growth order of \(T_{\varepsilon}\). Assuming \(\varepsilon\) is sufficiently small, \(T_{\varepsilon}\) becomes large and satisfies \(T_{\varepsilon} > t_0\). To remain within the \(\varepsilon\)-neighborhood of \(\boldsymbol{\theta}^\star\), it is sufficient to have:
$2 C_{T_\varepsilon} \leq \varepsilon$,
which using the fact that $T_{\varepsilon}> t_0$, it reduces to
\begin{align}
 \label{eq:ineq1}
2\frac{\mu^{T_{\varepsilon}}-1}{\mu-1}  \Big[  {\frac{2\lambda \gamma ^{2^{{T_{\varepsilon}}-t_0}}}{L (1-\gamma ^{2^{{T_{\varepsilon}}-t_0}})}}
+  \frac{\delta+1}{\lambda} \frac{\sigma_{n} \big\| \bold{c}_{T_{\varepsilon}} \big\|}{|\mathcal{D}| \sqrt{\alpha^{\text{ZF}}_{T_{\varepsilon}}}} \Big] \leq \varepsilon .
\end{align}
Since $\mu < 1$, as ${T_{\varepsilon}}$ grows large $\mu^{T_{\varepsilon}} \to 0$; therefore, the left-hand-side coefficient tends to ${2}/({1-\mu})$. Noting that $\gamma \in [0,\frac{1}{2}]$, we can further replace the first term inside the brackets with ${2\lambda \gamma ^{2^{T_{\varepsilon}-t_0}}}/{L}$ for large $T_{\varepsilon}$.
By substituting into \eqref{eq:ineq1}, and taking $\log$ from both sides we finally have
\begin{align} \label{eq:t-t0}
\log(\frac{1}{\varepsilon}) \leq - \log(\frac{4\lambda}{L-\frac{L^2 \delta}{\lambda}}) -   2^{T_{\varepsilon}-t_0} \log(\gamma).  
\end{align}
As $\log(\gamma)<0$, the second term on the right-hand-side of \eqref{eq:t-t0} is positive. Noting that the first term does not scale, we can finally write $T_{\varepsilon} \leq \mathcal{O} \left( \log \log \frac{1}{\varepsilon}\right)$, which concludes the proof.

\section{Proof of Corollary \ref{corr:th2}} \label{app:corproof2}
The proof is similar to that of Corollary~\ref{corr:th1}: from Theorem~\ref{eq:thconv}, we know that if $\mathbb{E} \big[ \big\| \boldsymbol{\theta}_{0}-\boldsymbol{\theta}^{\star} \big\| \big] \geq  C_t$; then, we have
\begin{align}
\mathbb{E} \big[ \|\boldsymbol{\theta}_{t}-\boldsymbol{\theta}^{\star} \|\big] 
& \leq 2 \mu^t \mathbb{E} \big[ \big\| \boldsymbol{\theta}_{0}-\boldsymbol{\theta}^{\star} \big\| \big].
\end{align}
Thus, to lie within $\varepsilon$-neighborhood of optimal model, it is sufficient for $T_{\varepsilon} > t_0$ to satisfy $2 \mu^{T_{\varepsilon}} \mathbb{E} \big[ \big\| \boldsymbol{\theta}_{0}-\boldsymbol{\theta}^{\star} \big\| \big] \leq \varepsilon$,  which is equivalent to 
\begin{align}
T_{\varepsilon} \log(\frac{1}{\mu}) \geq \log \left( \frac{2 \mathbb{E} \big[ \big\| \boldsymbol{\theta}_{0}-\boldsymbol{\theta}^{\star} \big\| \big]}{\varepsilon} \right).
\end{align}
This concludes the proof.

\section{Proof of Theorem \ref{th:lossconv}} \label{app:thf}
As the global loss is $L$-Lipschitz gradient, we have
\begin{align} \label{eq:L-smooth}
f(\boldsymbol{\theta}_t) - f(\boldsymbol{\theta}^{\star}) - \nabla f^{\mathsf{T}}(\boldsymbol{\theta}^{\star}) (\boldsymbol{\theta}_t - \boldsymbol{\theta}^{\star}) \leq \frac{L}{2} \| \boldsymbol{\theta}_t- \boldsymbol{\theta}^{\star}\|^2. 
\end{align}
Taking expectation from both sides, and noting that $f(\boldsymbol{\theta}^{\star})=\boldsymbol{0}$, we obtain
\begin{align} \label{eq:L-smooth2}
\mathbb{E} \big[  f (\boldsymbol{\theta}_{t}) \big] - f (\boldsymbol{\theta}^{\star})  \leq  \frac{L}{2}    \mathbb{E} \big[ \| \boldsymbol{\theta}_t- \boldsymbol{\theta}^{\star}\|^2 \big].
\end{align}
This implies that to bound $\mathbb{E} \big[  f (\boldsymbol{\theta}_{t}) \big] - f (\boldsymbol{\theta}^{\star})$, we should bound $\mathbb{E} \big[ \| \boldsymbol{\theta}_t- \boldsymbol{\theta}^{\star}\|^2 \big]$. From \eqref{eq:app1}, we have  
\begin{align}
\mathbb{E} \big[ \| \boldsymbol{\theta}_t- \boldsymbol{\theta}^{\star}\|^2 \big] & \leq \mathbb{E} \big[\mathcal{M}_1^2\big] + \mathbb{E} \big[ \mathcal{M}_2^2\big].
\end{align}
Using the bounds in \eqref{eq:polyak} to \eqref{eq:noiseboumd}, we obtain
\begin{align} \label{eq:boundfin2}
\mathbb{E} \big[ \|\boldsymbol{\theta}_{t+1}-\boldsymbol{\theta}^{\star} \|^2\big]\leq \mu^2 \mathbb{E} \big[ \big\| \boldsymbol{\theta}_{t}-\boldsymbol{\theta}^{\star} \big\|^2 \big]  + C'_t,
\end{align}
where $C'_t$ is given by
\begin{align}
C'_t =\begin{cases}
 C^{\prime}_0 ~ &t\leq t_0 \\
 C^{\prime}_1 &t > t_0
\end{cases},
\end{align}
with $C^{\prime}_0$ and $C^{\prime}_1$ being defined in \cref{th:lossconv}. Using \eqref{eq:boundfin2} recursively, we next obtain
\begin{align}\label{eq:boundfin3}
\mathbb{E} \big[ \|\boldsymbol{\theta}_{t+1}-\boldsymbol{\theta}^{\star} \|^2 \big] &\leq \mu^{2t} \mathbb{E} \big[ \big\| \boldsymbol{\theta}_{0}-\boldsymbol{\theta}^{\star} \big\|^2 \big] \nonumber \\
&+ \frac{\mu^{2t}-1}{\mu^2-1}
\begin{cases}
 C^{\prime}_0, ~ &t\leq t_0 \\
 C^{\prime}_1, &t > t_0
\end{cases}
\end{align}
Finally, using \eqref{eq:boundfin3} in \eqref{eq:L-smooth2} the proof is concluded.

\bibliography{main}
\bibliographystyle{IEEEtran}

\end{document}